\documentclass[colorlinks]{article}
\usepackage[T1]{fontenc}
\usepackage[utf8]{inputenc}
\usepackage{color}
\usepackage{wrapfig}
\usepackage{bm}
\usepackage{amsmath}
\usepackage{amsthm}
\usepackage{graphicx}
\usepackage[numbers]{natbib}
\usepackage{microtype}
\usepackage[unicode=true,
 bookmarks=true,bookmarksnumbered=true,bookmarksopen=true,bookmarksopenlevel=1,
 breaklinks=false,pdfborder={0 0 0},pdfborderstyle={},backref=false,colorlinks=true]
 {hyperref}

\makeatletter

\providecommand{\tabularnewline}{\\}

\theoremstyle{plain}
\newtheorem{thm}{\protect\theoremname}
\theoremstyle{definition}
\newtheorem{defn}[thm]{\protect\definitionname}
\theoremstyle{plain}
\newtheorem{lem}[thm]{\protect\lemmaname}

\usepackage{xcolor}
\definecolor{mycol}{rgb}{0,0,0.65}
\hypersetup{
    colorlinks,
    linkcolor={mycol},
    citecolor={mycol},
    urlcolor={mycol}
}

\usepackage[final]{neurips_2019}

\usepackage{url}% simple URL typesetting
\usepackage{amsfonts}% blackboard math symbols
\usepackage{nicefrac}% compact symbols for 1/2, etc.

\author{Justin Domke\\
College of Information and Computer Sciences\\
University of Massachusetts Amherst\\
\texttt{domke@cs.umass.edu}
}

\usepackage{algorithm}
\usepackage{wrapfig}

\usepackage{thmtools}
\usepackage{thm-restate}

\usepackage[shortlabels]{enumitem}
\setlist[itemize]{leftmargin=18pt}

\DeclareMathAlphabet{\mathbfsf}{\encodingdefault}{\sfdefault}{bx}{n}
\newcommand{\upgreektemplate}[2]{#2{
\renewcommand{\alpha}{\upalpha}
\renewcommand{\beta}{\upbeta}
\renewcommand{\theta}{\uptheta}
\renewcommand{\gamma}{\upgamma}
\renewcommand{\lambda}{\uplambda}
\renewcommand{\delta}{\updelta}
\renewcommand{\phi}{\upphi}
\renewcommand{\zeta}{\upzeta}
\renewcommand{\Lambda}{\Uplambda}
\renewcommand{\Gamma}{\Upgamma}
\renewcommand{\Delta}{\Updelta}
\renewcommand{\Theta}{\Uptheta}
#1
}}
\newcommand{\upgreek}[1]{\upgreektemplate{#1}{\mathsf}}
\newcommand{\bupgreek}[1]{\upgreektemplate{#1}{\mathbfsf}}
\usepackage{upgreek}

\usepackage[nameinlink,capitalize]{cleveref}
\AtBeginDocument{%
\let\ref\cref
}
\Crefname{equation}{Eq.}{Eqs.}
\Crefname{fig1}{Fig.}{Figs.}
\Crefname{lem1}{Lem.}{Lems.}
\Crefname{thm1}{Thm.}{Thms.}
\Crefname{section}{Sec.}{Secs.}

\makeatother

\providecommand{\definitionname}{Definition}
\providecommand{\lemmaname}{Lemma}
\providecommand{\theoremname}{Theorem}

\begin{document}
\global\long\def\argmin{\operatornamewithlimits{argmin}}%

\global\long\def\argmax{\operatornamewithlimits{argmax}}%

\global\long\def\prox{\operatornamewithlimits{prox}}%

\global\long\def\diag{\operatorname{diag}}%

\global\long\def\lse{\operatorname{lse}}%

\global\long\def\R{\mathbb{R}}%

\global\long\def\E{\operatornamewithlimits{\mathbb{E}}}%

\global\long\def\P{\operatornamewithlimits{\mathbb{P}}}%

\global\long\def\V{\operatornamewithlimits{\mathbb{V}}}%

\global\long\def\N{\mathcal{N}}%

\global\long\def\L{\mathcal{L}}%

\global\long\def\C{\mathbb{C}}%

\global\long\def\tr{\operatorname{tr}}%

\global\long\def\norm#1{\left\Vert #1\right\Vert }%

\global\long\def\norms#1{\left\Vert #1\right\Vert ^{2}}%

\global\long\def\pars#1{\left(#1\right)}%

\global\long\def\pp#1{(#1)}%

\global\long\def\bracs#1{\left[#1\right]}%

\global\long\def\bb#1{[#1]}%

\global\long\def\verts#1{\left\vert #1\right\vert }%

\global\long\def\vv#1{\vert#1\vert}%

\global\long\def\Verts#1{\left\Vert #1\right\Vert }%

\global\long\def\VV#1{\Vert#1\Vert}%

\global\long\def\angs#1{\left\langle #1\right\rangle }%

\global\long\def\KL#1{[#1]}%

\global\long\def\KL#1#2{KL\pars{#1\middle\Vert#2}}%

\global\long\def\div{\text{div}}%

\global\long\def\erf{\text{erf}}%

\global\long\def\vvec{\text{vec}}%

\global\long\def\b#1{\bm{#1}}%

\global\long\def\r#1{\upgreek{#1}}%

\global\long\def\br#1{\bupgreek{\bm{#1}}}%

\global\long\def\T{\mathcal{T}}%

\global\long\def\ep{\bm{\varepsilon}}%

\global\long\def\rep{\bm{\upvarepsilon}}%

\global\long\def\marker{\checkmark}%

\global\long\def\zo{\bar{\b z}}%

\global\long\def\wo{\bar{\b w}}%

\global\long\def\locscale{\mathrm{LocScale}}%

\global\long\def\w{\b w}%

\global\long\def\v{\b v}%

\global\long\def\wr{\br w}%

\global\long\def\z{\b z}%

\global\long\def\x{\b x}%

\global\long\def\zr{\r z}%

\global\long\def\u{\b u}%

\global\long\def\ur{\r u}%

\global\long\def\gr{\r g}%

\title{Provable Gradient Variance Guarantees for Black-Box Variational Inference}

\maketitle
\maketitle
\begin{abstract}
Recent variational inference methods use stochastic gradient estimators
whose variance is not well understood. Theoretical guarantees for
these estimators are important to understand when these methods will
or will not work. This paper gives bounds for the common ``reparameterization''
estimators when the target is smooth and the variational family is
a location-scale distribution. These bounds are unimprovable and thus
provide the best possible guarantees under the stated assumptions.
\end{abstract}

\section{Introduction}

Take a distribution $p\pp{\z,\x}$ representing relationships between
data $\x$ and latent variables $\z$. After observing $\x$, one
might wish to approximate the marginal probability $p\pp{\x}$ or
the posterior $p\pp{\z|\x}.$ Variational inference (VI) is based
on the simple observation that for any distribution $q\pp{\z},$
\begin{equation}
\log p\pp{\x}=\underbrace{\E_{\zr\sim q}\log\frac{p\pp{\zr,x}}{q\pp{\zr}}}_{\mathrm{ELBO}\pp q}+\KL{q\pp{\zr}}{p\pp{\zr|\x}}.\label{eq:elbo-decomp}
\end{equation}
VI algorithms typically choose an approximating family $q_{\w}$ and
maximize $\mathrm{ELBO}\pp{q_{\w}}$ over $\w$. Since $\log p\pp{\x}$
is fixed, this simultaneously tightens a lower-bound on $\log p\pp{\x}$
and minimizes the divergence from $q_{\w}\pp{\z}$ to the posterior
$p\pp{\z|\x}$.

Traditional VI algorithms suppose $p$ and $q_{\w}$ are simple enough
for certain expectations to have closed forms, leading to deterministic
coordinate-ascent type algorithms \citep{Ghahramani_2001_PropagationAlgorithmsVariational,Blei_2017_VariationalInferenceReview,Winn_2005_VariationalMessagePassing}.
Recent work has turned towards stochastic optimization. There are
two motivations for this. First, stochastic data subsampling can give
computational savings \citep{Hoffman_2013_StochasticVariationalInference}.
Second, more complex distributions can be addressed if $p$ is treated
as a ``black box'', with no expectations available \citep{Ranganath_2014_BlackBoxVariational,Salimans_2013_FixedFormVariationalPosterior,Wingate_2013_AutomatedVariationalInference}.
In both cases, one can still estimate a \emph{stochastic} gradient
of the ELBO \citep{Titsias_2014_DoublyStochasticVariational} and
thus use stochastic gradient optimization. It is possible to address
very complex and large-scale problems using this strategy \citep{Regier_2016_LearningAstronomicalCatalog}.

These improvements in scale and generality come at a cost: Stochastic
optimization is typically less reliable than deterministic coordinate
ascent. Convergence is often a challenge, and methods typically use
heuristics for parameters like step-sizes. Failures do frequently
occur in practice \citep{Yao_2018_YesDidIt,Regier_2017_FastBlackboxVariationala,Fan_2015_FastSecondOrderStochastic}.

To help understand when black-box VI can be expected to work, this
paper investigates the variance of gradient estimates. This is a major
issue in practice, and many ideas have been proposed to attempt to
reduce the variance \citep{Miller_2017_ReducingReparameterizationGradient,Geffner_2018_UsingLargeEnsembles,Roeder_2017_StickingLandingSimple,Buchholz_2018_QuasiMonteCarloVariational,Titsias_2015_LocalExpectationGradients,Ruiz_2016_GeneralizedReparameterizationGradient,Ruiz_2016_OverdispersedBlackBoxVariational,Tan_2018_Gaussianvariationalapproximation}.
 Despite all this, few rigorous guarantees on the variance of gradient
estimators seem to be known (\ref{subsec:Related-work}).

\subsection{Contributions}

This paper studies ``reparameterization'' (RP) or ``path'' based
gradient estimators when $q_{\w}$ is in a multivariate location-scale
family. We decompose $\mathrm{ELBO}\pp{q_{\w}}=l\pp{\w}+h\pp{\w}$
where $h\pp{\w}$ is the entropy of $q_{\w}$ (known in closed-form)
and $l\pp{\w}=\E_{\zr\sim q_{\w}}\log p\pp{\z,\x}.$ The key assumption
is that $\log p\pp{\z,\x}$ is (Lipschitz) \emph{smooth} as a function
of $\z$, meaning that $\nabla_{\z}\log p\pp{\z,\x}$ can't change
too quickly as $\z$ changes. Formally $f\pp{\z}$ is $M$-smooth
if $\VV{\nabla f\pp{\z}-\nabla f\pp{\z'}}_{2}\leq M\VV{\z-\z'}_{2}.$
\begin{description}
\item [{Bound~for~smooth~target~distributions:}] If $\gr$ is the RP
gradient estimator of $\nabla l\pp{\w}$ and $\log p$ is $M$-smooth,
then $\E\VV{\gr}^{2}$ is bounded by a quadratic function of $\w$
(\ref{thm:batch_f_scalar_M}). With a small relaxation, this is $\E\VV{\gr}^{2}\leq aM^{2}\VV{\w-\bar{\w}}^{2}$
(\ref{eq:thm1_user_friendly}) where $\bar{\w}$ are fixed parameters
and $a$ is determined by the location-scale family.
\item [{Generalized~bound:}] We extend this result to consider a more
general notion of ``matrix'' smoothness (\ref{thm:batch_f_matrix_M})
reflecting that the sensitivity of $\nabla_{\z}\log p\pp{\z,\x}$
to changes in $\z$ may depend on the direction of change.
\item [{Data~Subsampling:}] We again extend this result to consider data
subsampling (\ref{thm:g2-stoch}). In particular, we observe that
\emph{non-uniform} subsampling gives tighter bounds.
\end{description}
In all cases, we show that the bounds are unimprovable. We experimentally
compare these bounds to empirical variance.

\section{Setup\label{sec:Setup}}

Given some ``black box'' function $f$, this paper studies estimating
gradients of functions $l$ of the form $l\pp{\w}=\E_{\zr\sim q_{\w}}f\pp{\zr}.$
Now, suppose some base distribution $s$ and mapping $\T_{\w}$ are
known such that if $\ur\sim s$, then $\T_{\w}\pp{\ur}\sim q_{\w}$.
Then, $l$ can be written as

\[
l\pp{\w}=\E_{\ur\sim s}f\pp{\T_{\w}\pp{\ur}}.
\]

If we define $\gr=\nabla_{\w}f\pp{\T_{\w}\pp{\ur}},$ then $\gr$
is an unbiased estimate of $\nabla l$, i.e. $\E\gr=\nabla l\pp{\w}.$
The same idea can be used when $f$ is composed as a finite sum as
$f\pp{\z}=\sum_{n=1}^{N}f_{n}\pp{\z}.$ If $N$ is large, even evaluating
$f$ once might be expensive. However, take any positive distribution
$\pi$ over $n\in\left\{ 1,\cdots,N\right\} $ and sample $\r n\sim\pi$
independently of $\ur$. Then, if we define $\r g=\nabla_{\w}\pi\pp{\r n}^{-1}f_{\r n}\pp{\T_{\w}\pp{\ur}}$,
this is again an unbiased estimator with $\E\r g=\nabla l\pp{\w}.$

Convergence rates in stochastic optimization depend on the variability
of the gradient estimator, typically either via the expected squared
norm (ESN) $\E\Vert\gr\Vert_{2}^{2}$ or the trace of the variance
$\tr\V\gr.$ These are closely related, since $\E\VV{\gr}_{2}^{2}=\tr\V\gr+\VV{\E\gr}_{2}^{2}.$

The goal of this paper is to bound the variability of $\gr$ for reparameterization
/ path estimators of $\gr$. This requires making assumptions about
(i) the transformation function $\T_{\w}$ and base distribution $s$
(which determine $q_{\w})$ and (ii) the target function $f$.

Here, we are interested in the case of affine mappings. We use the
mapping\citep{Titsias_2014_DoublyStochasticVariational}

\[
\T_{\w}\pp{\u}=C\u+\b m,
\]
where $\w=\pp{\b m,C}$ is a single vector of all parameters. This
is the most common mapping used to represent location-scale families.
That is, if $\ur\sim s$ then $\T_{\w}\pp{\ur}$ is equal in distribution
to a location-scale family distribution. For example, if $s=\N(0,I)$
then $\T_{\w}\pp{\ur}$ is equal in distribution to $\N\pp{\b m,CC^{\top}}.$

We will refer to the base distribution as \textbf{standardized} if
the components of $\ur=\pp{\ur_{1},\cdots,\ur_{d}}\sim s$ are iid
with $\E\ur_{1}=\E\ur_{1}^{3}=0$ and $\V\ur_{1}=1.$ The bounds will
depend on the fourth moment $\kappa=\E\bb{\ur_{1}^{4}},$ but are
otherwise independent of $s$.

To apply these estimators to VI, choose $f\pp{\z}=\log\pp{\z,\x}$.
Then $\mathrm{ELBO}\pp{\w}=l\pp{\w}+h\pp{\w}$ where $h$ is the entropy
of $q_{\w}$. Stochastic estimates of the gradient of $l$ can be
employed in a stochastic gradient method to maximize the ELBO. To
model the stochastic setting, suppose that $X=\pp{\x_{1},\cdots,\x_{N}}$
are iid and $p\pp{\z,X}=p\pp{\z}\prod_{n=1}^{N}p\pp{\x_{n}\vert\z}.$
Then, one may choose, e.g. $f_{n}\pp{\z}=\frac{1}{N}\log p\pp{\z}+\log p\pp{\x_{n}|\z}.$
The entropy $h$ is related to the (constant) entropy of the base
distribution as $h\pp{\w}=\mathrm{Entropy}\pp s+\log\vv C$.

The main bounds of this paper concern estimators for the gradient
of $l\pp{\w}$ alone, disregarding $h\pp{\w}.$ There are two reasons
for this. First, in location-scale families, the exact gradient of
$h\pp{\w}$ is known. Second, if one uses a stochastic estimator for
$h\pp{\w},$ this can be ``absorbed'' into $l\pp{\w}$ to some degree.
This is discussed further in \ref{sec:Discussion}.

\section{Variance Bounds}

\subsection{Technical Lemmas}

We begin with two technical lemmas which will do most of the work
in the main results. Both have (somewhat laborious) proofs in \ref{sec:Proofs}
(Appendix). The first lemma relates the norm of the \emph{parameter}
gradient of $f\pp{\T_{\b w}\pp{\u}}$ (with respect to $\w$) to the
norm of the gradient of $f\pp{\z}$ itself, evaluated at $\z=\T_{\w}\pp{\u}.$

\begin{restatable}{lem1}{gradnorm}\label{lem:gradnorm}For any $\w$
and $\u$, $\Verts{\nabla_{\b w}f\pp{\T_{\b w}\pp{\u}}}_{2}^{2}=\Verts{\nabla f\pp{\T_{\b w}\pp{\u}}}_{2}^{2}\pars{1+\Verts{\u}_{2}^{2}}.$
\end{restatable}

The proof is tedious but essentially amounts to calculating the derivative
with respect to each component of $\w$ (i.e. entries $\b m_{i}$
and $C_{ij}$), summing the square of all entries, and simplifying.
The second lemma gives a closed-form for the expectation of a closely
related expression that will appear in the proof of \ref{thm:batch_f_scalar_M}
as a consequence of applying \ref{lem:gradnorm}.

\begin{restatable}{lem1}{expectedtminuszstar}\label{lem:expectedtminuszstar}Let
$\ur\sim s$ for $s$ standardized with $\ur\in\R^{d}$ and $\E_{\ur\sim s}\ur_{i}^{4}=\kappa$.
Then for any $\zo,$
\[
\E\VV{\T_{\b w}\pp{\ur}-\zo}_{2}^{2}\pars{1+\VV{\ur}_{2}^{2}}=\pars{d+1}\VV{\b m-\zo}_{2}^{2}+\pars{d+\kappa}\VV C_{F}^{2}.
\]

\end{restatable}

Again, the proof is tedious but based on simple ideas: Substitute
the definition of $\T_{\w}$ into the left-hand side and expand all
terms. This gives terms between zeroth and fourth order (in $\ur$).
Calculating the exact expectation of each and simplifying using the
assumption that $s$ is standardized gives the result.

\subsection{Basic Variance Bound\label{sec:Variance-batch}}

Given these two lemmas, we give our major technical result, bounding
the variability of a reparameterization-based gradient estimator.
This will be later be extended to consider data subsampling, and a
generalized notion of smoothness. Note that we do \emph{not} require
that $f$ be convex.

\begin{restatable}{thm1}{batchfscalarM}Suppose $f$ is $M$-smooth,
$\zo$ is a stationary point of $f$, and $s$ is standardized with
$\ur\in\R^{d}$ and $\E\ur_{i}^{4}=\kappa$. Let $\gr=\nabla_{\b w}f\pars{\T_{\b w}\pp{\ur}}$
for $\ur\sim s$. Then,\emph{\label{thm:batch_f_scalar_M}}
\begin{equation}
\E\Verts{\gr}_{2}^{2}\leq M^{2}\pars{\pp{d+1}\Verts{\b m-\zo}_{2}^{2}+\pp{d+\kappa}\Verts C_{F}^{2}}.\label{eq:thm1_result}
\end{equation}
Moreover, this result is unimprovable without further assumptions.\end{restatable}
\begin{proof}
We expand the definition of $\gr$, and use the above lemmas and the
smoothness of $f.$

\begin{align*}
\ \ \ \ \ \ \ \ \ \ \ \ \E\Verts{\gr}_{2}^{2} & =\E\Verts{\nabla_{\b w}f\pp{\T_{\b w}\pp{\ur}}}_{2}^{2} & \text{ (Definition of \ensuremath{\gr})}\\
 & =\E\Verts{\nabla f\pp{\T_{\b w}\pp{\ur}}}_{2}^{2}\pp{1+\Verts{\ur}_{2}^{2}} & \text{ (\ref{lem:gradnorm})}\\
 & =\E\Verts{\nabla f\pp{\T_{\b w}\pp{\ur}}-\nabla f\pp{\zo}}_{2}^{2}\pp{1+\Verts{\ur}_{2}^{2}} & \text{ (\ensuremath{\nabla f\pp{\zo}=0})}\\
 & \leq\E M^{2}\Verts{\T_{\b w}\pp{\ur}-\zo}_{2}^{2}\pp{1+\Verts{\ur}_{2}^{2}} & \text{ (\ensuremath{f} is smooth)}\\
 & =M^{2}\pars{\pars{d+1}\Verts{\b m-\zo}_{2}^{2}+\pars{d+\kappa}\Verts C_{F}^{2}}. & \text{ (\ref{lem:expectedtminuszstar})}
\end{align*}

To see that this is unimprovable without further assumptions, observe
that the only inequality is using the smoothness on $f$ to bound
the norm of the difference of gradients at $\T_{\b w}\pp u$ and at
$\zo$. But for $f\pp{\z}=\frac{M}{2}\Verts{\z-\zo}_{2}^{2}$ this
inequality is tight. Thus, for any $M$ and $\zo$, there is a function
$f$ satisfying the assumptions of the theorem such that \ref{eq:thm1_result}
is an equality.
\end{proof}
With a small amount of additional looseness, we can cast \ref{eq:thm1_result}
into a more intuitive form. Define $\wo=\pp{\zo,0_{d\times d}}$,
where $0_{d\times d}$ is a $d\times d$ matrix of zeros. Then, $\VV{\w-\bar{\w}}_{2}^{2}=\VV{\b m-\bar{\z}}_{2}^{2}+\VV C_{F}^{2}$,
so we can slightly relax \ref{eq:thm1_result} to the more user-friendly
form of
\begin{equation}
\E\Verts{\gr}_{2}^{2}\leq\pp{d+\kappa}M^{2}\Verts{\w-\wo}_{2}^{2}.\label{eq:thm1_user_friendly}
\end{equation}
The only additional looseness is bounding $d+1\leq d+\kappa$. This
is justified since when $s$ is standardized, $\kappa=\ur_{i}^{4}$
is the kurtosis, which is at least one. Here, $\kappa$ is determined
by $s$ and does not depend on the dimensionality. For example, if
$s$ is Gaussian, $\kappa=3$. Thus, \ref{eq:thm1_user_friendly}
will typically not be much looser than \ref{eq:thm1_result}.

Intuitively, $\wo$ are parameters that concentrate $q$ entirely
at a stationary point of $f$. It is not hard to show that $\VV{\w-\wo}^{2}=\E_{\zr\sim q_{\w}}\VV{\zr-\zo}^{2}.$
Thus, \ref{eq:thm1_user_friendly} intuitively says that $\E\VV{\gr}^{2}$
is bounded in terms of how far far the average point sampled from
$q_{\w}$ is from $\zo$. Since $f$ need not be convex, there might
be multiple stationary points. In this case, \ref{thm:batch_f_scalar_M}
holds simultaneously for all of them.

\subsection{Generalized Smoothness}

Since the above bound is not improvable, tightening it requires stronger
assumptions. The tightness of \ref{thm:batch_f_scalar_M} is determined
by the smoothness condition that the difference of gradients at two
points is bounded as $\Verts{\nabla f\pp y-\nabla f\pp z}_{2}\leq M\Verts{y-z}_{2}$.
For some problems, $f$ may be much smoother in \emph{certain directions}
than others. In such cases, the smoothness constant $M$ will need
to reflect the worst-case direction. To produce a tighter bound for
such situations, we generalize the notion of smoothness to allow $M$
to be a symmetric matrix. 
\begin{defn}
$f$ is $M$-\textbf{matrix-smooth} if $\Verts{\nabla f\pp{\b y}-\nabla f\pp{\z}}_{2}\leq\Verts{M\pp{\b y-\z}}_{2}$
(for symmetric $M$).
\end{defn}

We can generalize the result in \ref{thm:batch_f_scalar_M} to functions
with this matrix-smoothness condition. The proof is very similar.
The main difference is that after applying the smoothness condition,
the matrix $M$ needs to be ``absorbed'' into the parameters $\w=\pp{\b m,C}$
before applying \ref{lem:expectedtminuszstar}.

\begin{restatable}{thm1}{batchfmatrixM}Suppose $f$ is $M$-matrix
smooth, $\zo$ is a stationary point of $f$, and $s$ is standardized
with $\ur\in\R^{d}$ and $\E\ur_{i}^{4}=\kappa$. Let $\gr=\nabla_{\b w}f\pars{\T_{\b w}\pp{\ur}}$
for $\ur\sim s$. Then,\label{thm:batch_f_matrix_M}
\begin{equation}
\E\Verts{\gr}_{2}^{2}\leq\pars{d+1}\Verts{M\pp{\b m-\zo}}_{2}^{2}+\pars{d+\kappa}\Verts{MC}_{F}^{2}.\label{eq:thm2_result}
\end{equation}
Moreover, this result is unimprovable without further assumptions.

\end{restatable}
\begin{proof}
The proof closely mirrors that of \ref{thm:batch_f_scalar_M}. Here,
given $\w=\pp{\b m,C},$ we define $\v=\pp{M\b m,MC},$ to be $\w$
with $M$ ``absorbed'' into the parameters.
\begin{alignat*}{2}
\E\Verts{\gr}_{2}^{2} & =\E\Verts{\nabla_{\b w}f\pp{\T_{\b w}\pp{\ur}}}_{2}^{2} & \text{ Definition of \ensuremath{\gr})}\\
 & =\E\Verts{\nabla f\pp{\T_{\b w}\pp{\ur}}}_{2}^{2}\pp{1+\Verts{\ur}_{2}^{2}} & \text{ (\ref{lem:gradnorm})}\\
 & =\E\Verts{\nabla f\pp{\T_{\b w}\pp{\ur}}-\nabla f\pp{\zo}}_{2}^{2}\pp{1+\Verts{\ur}_{2}^{2}} & \text{ (\ensuremath{\nabla f\pp{\zo}=0})}\\
 & \leq\E\Verts{M\pars{\T_{\b w}\pp{\ur}-\zo}}_{2}^{2}\pp{1+\Verts{\ur}_{2}^{2}} & \text{ (\ensuremath{f} is smooth)}\\
 & =\E\Verts{\T_{\b v}\pp{\ur}-M\pp{\zo-\b m}}_{2}^{2}\pp{1+\Verts{\ur}_{2}^{2}} & \text{(Absorb \ensuremath{M} into \ensuremath{\v})}\\
 & =\pars{d+1}\Verts{M\b m-M\zo}_{2}^{2}+\pars{d+\kappa}\Verts{MC}_{F}^{2}. & \text{ (\ref{lem:expectedtminuszstar})}
\end{alignat*}
To see that this is unimprovable, observe that the only inequality
is the matrix-smoothness condition on $f$. But for $f\pp{\z}=\frac{1}{2}\pp{\z-\zo}^{\top}M\pp{\z-\zo},$
the difference of gradients $\VV{\nabla f\pp{\b y}-\nabla f\pp{\z}}_{2}=\VV{M\pp{\b y-\b z}}_{2}$
is an equality. Thus, for any $M$ and $\zo$, there is an $f$ satisfying
the assumptions of the theorem such that the bound in \ref{eq:thm2_result}
is an equality.
\end{proof}
It's easy to see that this reduces to \ref{thm:batch_f_scalar_M}
in the case that $f$ is smooth in the standard sense-- this corresponds
to the situation where $M$ is some constant times the identity. Alternatively,
one can simply observe that the two results are the same if $M$ is
a scalar. Thus, going forward we will use \ref{eq:thm2_result} to
represent the result with either type of smoothness assumption on
$f.$

\subsection{Subsampling}

Often, the function $f\pp{\z}$ takes the form of a sum over other
functions $f_{n}\pp{\z}$, typically representing different data.
Write this as
\[
f\pp{\z}=\sum_{n=1}^{N}f_{n}\pp{\z}.
\]

To estimate the gradient of $\E_{\ur\sim s}f\pp{\T_{\w}\pp{\ur}}$,
one can save time by using ``subsampling'': That is, draw a random
$n$, and then estimate the gradient of $\E_{\ur\sim s}f_{n}\pp{\T_{\w}\pp{\ur}}$.
The following result bounds this procedure. It essentially just takes
a set of estimators, one corresponding to each function $f_{n}$,
bounds their expected squared norm using the previous theorems, and
then combines these.

\begin{restatable}{thm1}{g2stoch}Suppose $f_{n}$ is $M_{n}$-matrix-smooth,
$\zo_{n}$ is a stationary point of $f_{n}$, and $s$ is standardized
with $\ur\in\R^{d}$ and $\E\ur_{i}^{4}=\kappa$. Let $\r g=\frac{1}{\pi\pp{\r n}}\nabla f_{\r n}\pp{\T_{\b w}\pp{\ur}}$
for $\ur\sim s$ and $\r n\sim\pi$ independent. Then,\label{thm:g2-stoch}

\begin{equation}
\E\Verts{\r g}_{2}^{2}\leq\sum_{n=1}^{N}\frac{1}{\pi\pp n}\pars{\pars{d+1}\Verts{M_{n}\pp{\b m-\zo_{n}}}_{2}^{2}+\pars{d+\kappa}\Verts{M_{n}C}_{F}^{2}}.\label{eq:g2_stoch}
\end{equation}
Moreover, this result is unimprovable without further assumptions.\end{restatable}
\begin{proof}
Consider a simple lemma: Suppose $\r a_{1}\cdots\r a_{N}$ are independent
random vectors and $\pi$ is any distribution over $1\cdots N.$ Let
$\r b=\r a_{\r n}/\pi\pp{\r n}$ for $\r n\sim\pi$, where $\r n$
is independent of $\r a_{n}.$ It is easy to show that $\E\r b=\sum_{n=1}^{N}\E\r a_{n}$
and $\E\Verts{\r b}_{2}^{2}=\sum_{n}\E\Verts{\r a_{n}}_{2}^{2}/\pi\pp n.$
The result follows from applying this with $\r a_{n}=\nabla_{\b w}f_{n}\pars{\T_{\b w}\pp{\ur}}$,
and then bounding $\E\Verts{\r a_{n}}_{2}^{2}$ using \ref{thm:batch_f_matrix_M}.

Again, in this result the only source of looseness is the use of the
smoothness bound for the component functions $f_{n}.$ Accordingly,
the result can be shown to be unimprovable: For any set of stationary
points $\zo$ and smoothness parameters $M_{n}$ we can construct
functions $f_{n}$ (as in \ref{thm:batch_f_matrix_M}) for which the
previous theorems are tight and thus this result is also tight.
\end{proof}
This result generalizes all the previous bounds: \ref{thm:batch_f_matrix_M}
is the special case when $N=1$, while \ref{thm:batch_f_scalar_M}
is the special-case when $N=1$ and $f_{1}$ is $M_{1}$-smooth (for
a scalar $M_{1})$. The case where $N>1$ but $f_{n}$ is $M_{n}$-smooth
(for scalar $M_{n}$) is also useful-- the bound in \ref{eq:g2_stoch}
remains valid, but with a scalar $M_{n}$.

\begin{figure}[t]
\scalebox{.95}{\includegraphics[viewport=0bp 25.2bp 252bp 144bp,clip,scale=0.59]{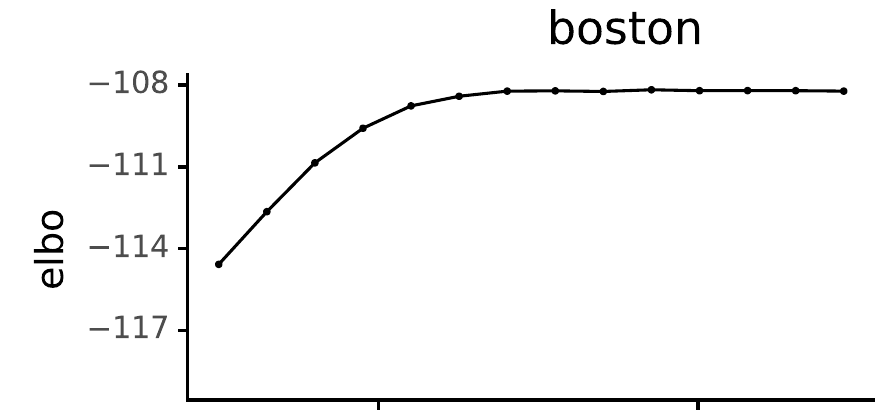}\hspace{-.1cm}\includegraphics[viewport=22.5bp 25.2bp 252bp 144bp,clip,scale=0.59]{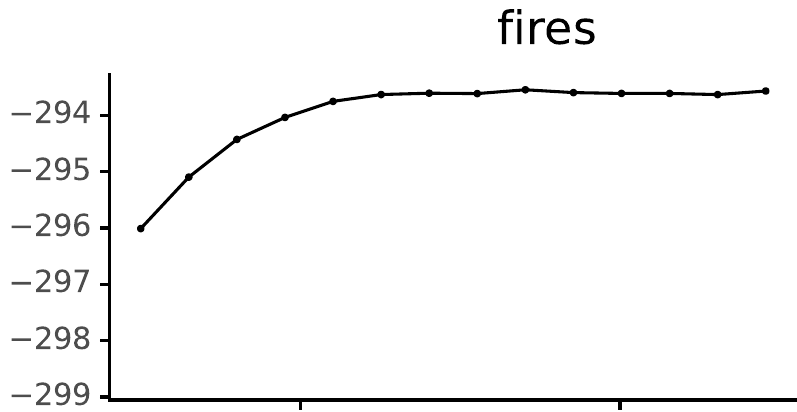}\hspace{-.1cm}\includegraphics[viewport=22.5bp 25.2bp 252bp 144bp,clip,scale=0.59]{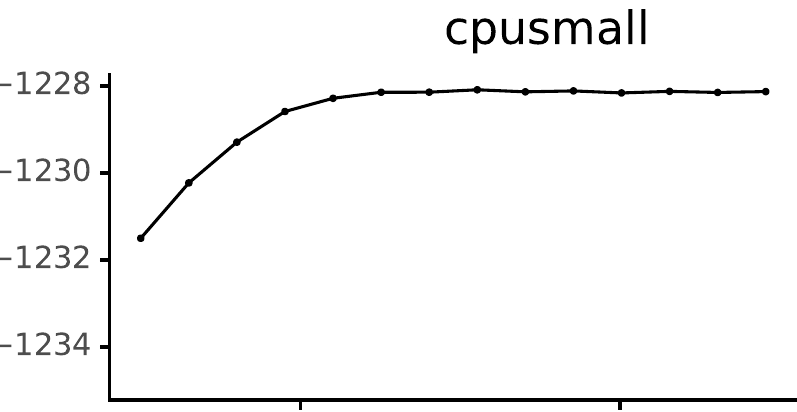}}

\scalebox{.95}{\includegraphics[viewport=0bp 25.2bp 252bp 126bp,clip,scale=0.59]{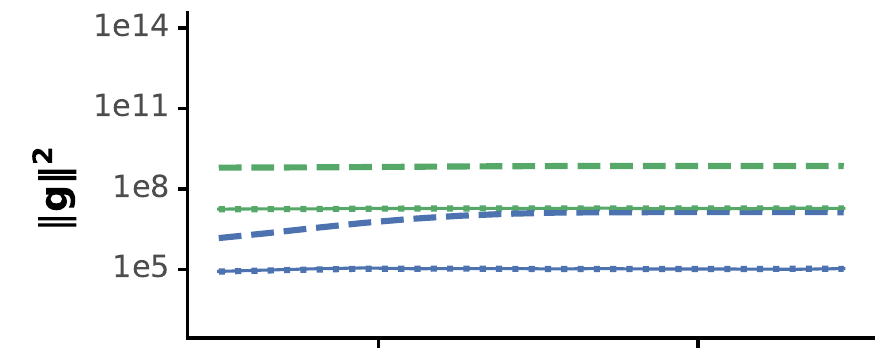}\hspace{-.1cm}\includegraphics[viewport=22.5bp 25.2bp 252bp 126bp,clip,scale=0.59]{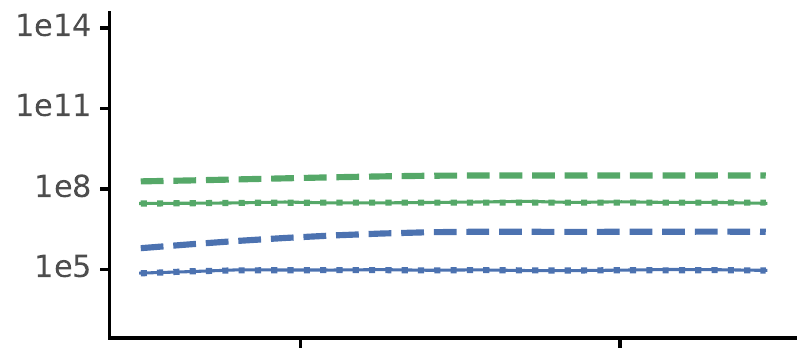}\hspace{-.1cm}\includegraphics[viewport=22.5bp 25.2bp 252bp 126bp,clip,scale=0.59]{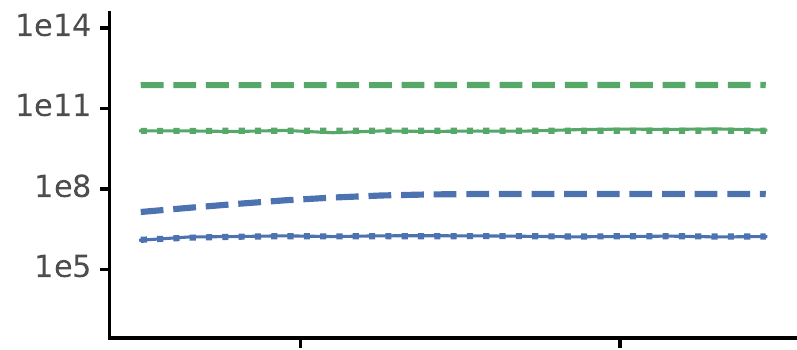}}\linebreak{}

\scalebox{.95}{\includegraphics[viewport=0bp 25.2bp 252bp 144bp,clip,scale=0.59]{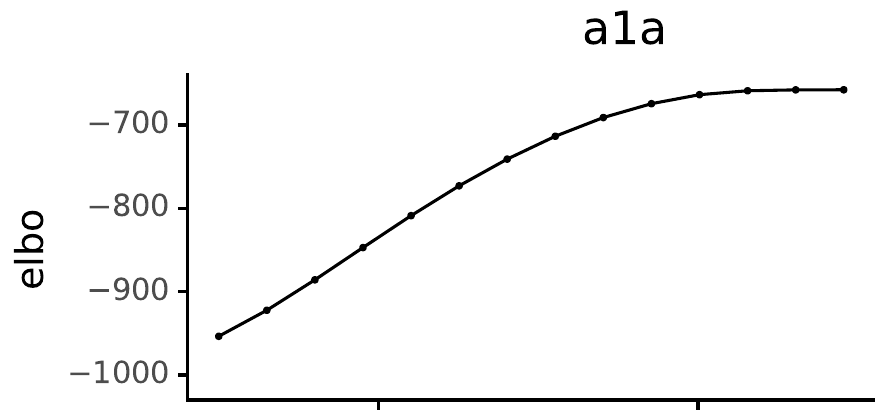}\hspace{-.1cm}\includegraphics[viewport=22.5bp 25.2bp 252bp 144bp,clip,scale=0.59]{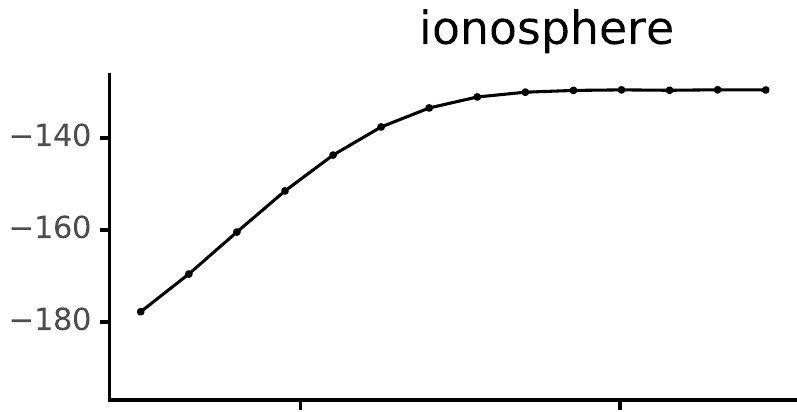}\hspace{-.1cm}\includegraphics[viewport=22.5bp 25.2bp 252bp 144bp,clip,scale=0.59]{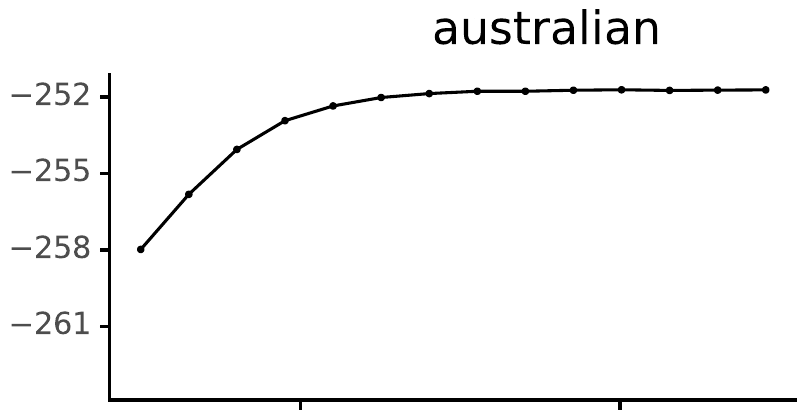}}

\scalebox{.95}{\includegraphics[viewport=0bp 25.2bp 252bp 126bp,clip,scale=0.59]{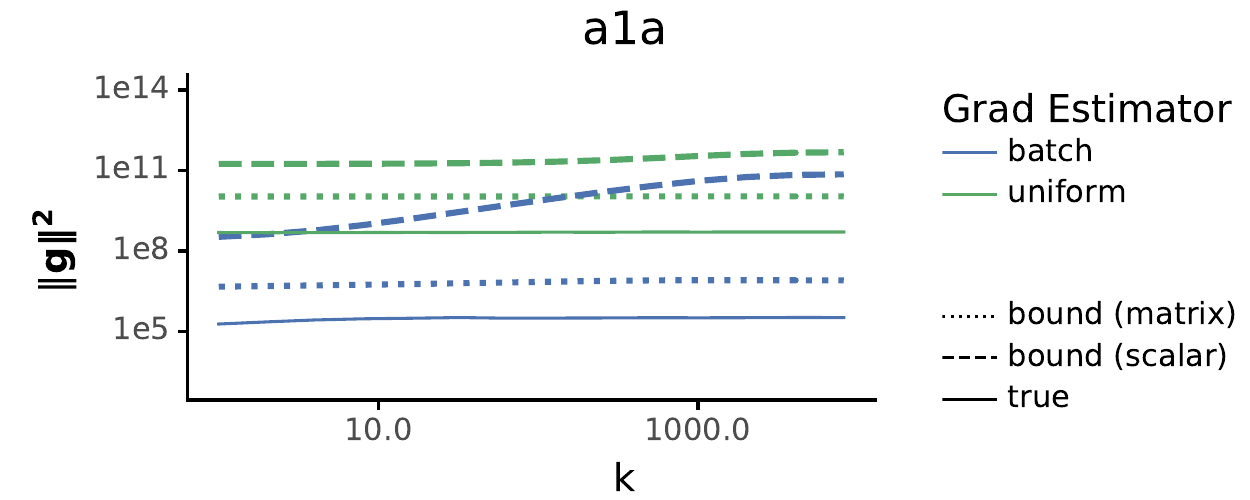}\hspace{-.1cm}\includegraphics[viewport=22.5bp 25.2bp 252bp 126bp,clip,scale=0.59]{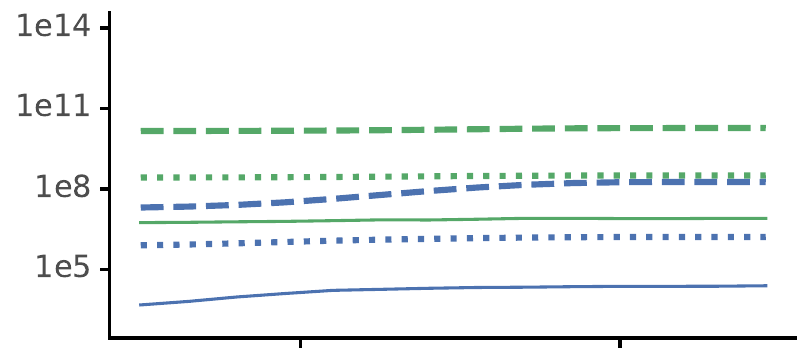}\hspace{-.1cm}\includegraphics[viewport=22.5bp 25.2bp 252bp 126bp,clip,scale=0.59]{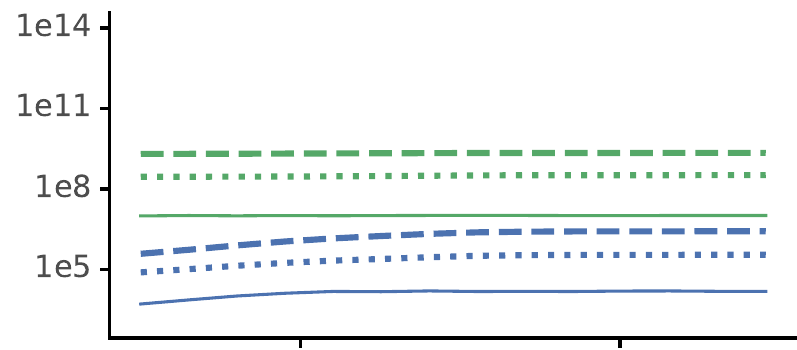}}\linebreak{}

\scalebox{.95}{\includegraphics[viewport=0bp 25.2bp 252bp 144bp,clip,scale=0.59]{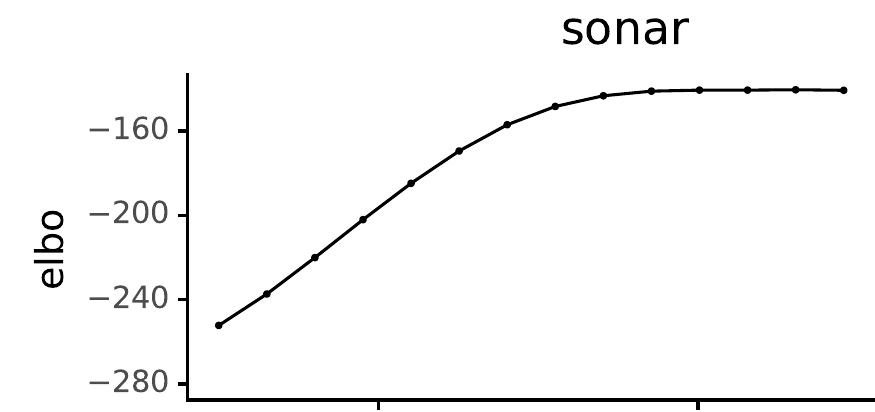}\hspace{-.1cm}\includegraphics[viewport=22.5bp 25.2bp 252bp 144bp,clip,scale=0.59]{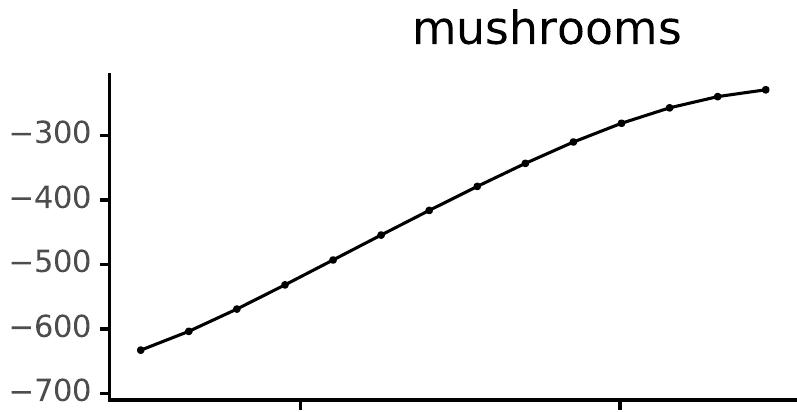}}

\scalebox{.95}{\includegraphics[viewport=0bp 0bp 252bp 126bp,clip,scale=0.59]{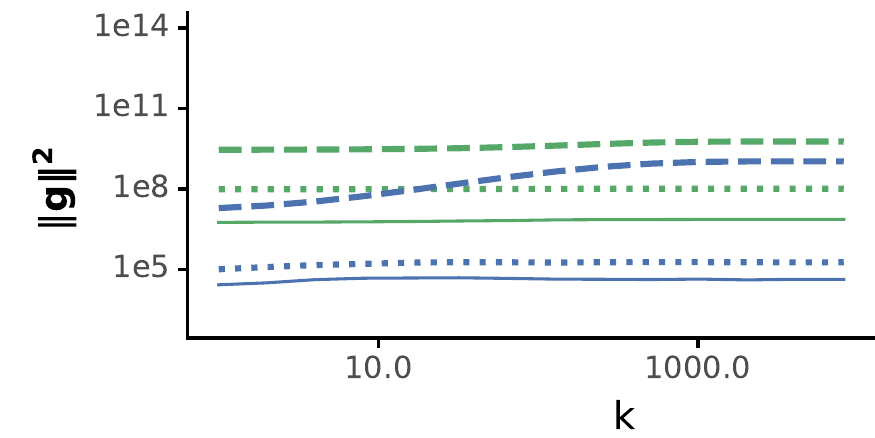}\hspace{-.1cm}\includegraphics[viewport=22.5bp 0bp 360bp 126bp,clip,scale=0.59]{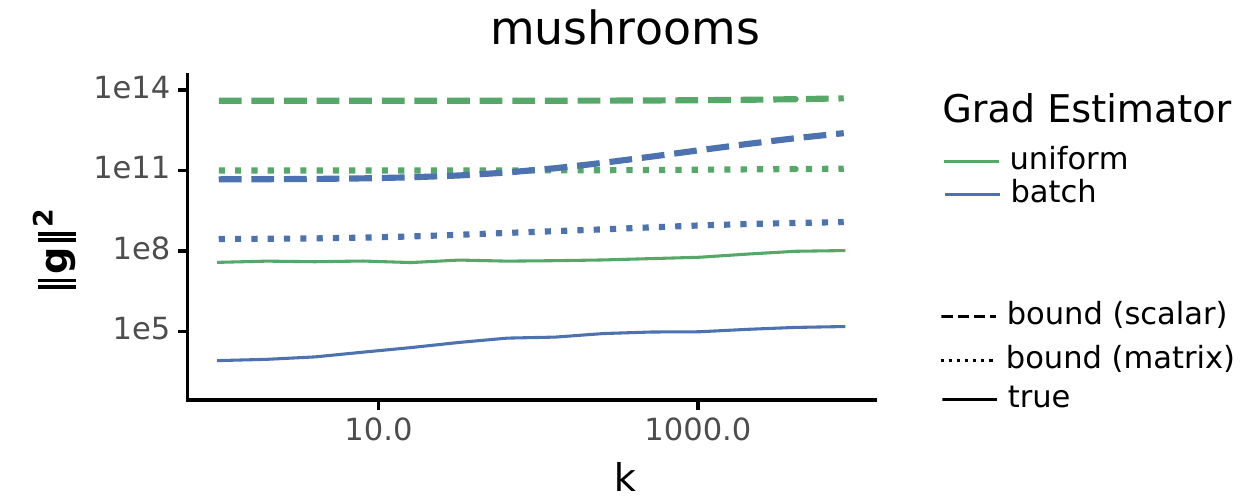}}

\caption{\textbf{How loose are the bounds compared to reality? }\emph{Odd Rows}:
Evolution of the ELBO during the single optimization trace used to
compare all estimators. \emph{Even Rows}: True and bounded variance
with gradients estimated in ``\texttt{batch}'' (using the full dataset
in each evaluation) and ``\texttt{uniform}'' (stochastically with
$\pi\protect\pp n=1/N$). The first two rows are for linear regression
models, while the rest are for logistic regression. \emph{Key Observations}:
(i) Batch estimation is lower-variance but higher cost (ii) variance
with stochastic estimation varies little over time (iii) using matrix
smoothness significantly tightens bounds -- and is exact for linear
regression models.\label{fig:traces}}
\end{figure}

\section{Empirical Evaluation}

\subsection{Model and Datasets}

We consider Bayesian linear regression and logistic regression models
on various datasets (\ref{tab:datasets}). Given data $\{\pp{\b x_{1},y_{1}},\cdots\pp{\b x_{N},y_{N}}\}$,
let $\b y$ be a vector of all $y_{n}$ and $X$ a matrix of all $\x_{n}.$
We assume a Gaussian prior so that $p\pp{\z,\b y|X}=\N\pp{\b z|0,\sigma^{2}I}\prod_{n=1}^{N}p\pp{y_{n}|\b z,\b x_{n}}.$
For linear regression, $p\pp{y_{n}|\b z,\b x_{n}}=\N\pp{y_{n}\vert\b z^{\top}\b x_{i},\rho^{2}}$,
while for logistic regression, $p\pp{y_{n}|\b z,\b x_{n}}=\mathrm{Sigmoid}\pp{y_{n}\b z^{\top}\b x_{n}}.$
For both models we use a prior of $\sigma^{2}=1.$ For linear regression,
we set $\rho^{2}=4.$

To justify the use of VI, apply the decomposition in \ref{eq:elbo-decomp}
substituting $p\pp{\z,\b y|X}$ in place of $p\pp{\z,\x}$ to get
that
\[
\log p\pp{\b y|X}=\E_{\zr\sim q}\log\frac{p\pp{\zr,\b y|X}}{q\pp{\zr}}+\KL{q\pp{\zr}}{p\pp{\zr|\b y,X}}.
\]

Thus, adjusting the parameters of $q$ to maximize the first term
on the right tightens a lower-bound on the conditional log likelihood
$\log p\pp{\b y|X}$ and minimizes the divergence from $q$ to the
posterior. So, we again take our goal as maximizing $l\pp{\w}+h\pp{\w}$.
In the batch setting, $f\pp{\z}=\log p\pp{\zr,\b y|X},$ while with
subsampling, $f_{n}\pp{\z}=\frac{1}{N}\log p\pp{\z}+\log p\pp{y_{n}\vert\z,\x_{n}}.$

\ref{sec:Smoothness-conditions-for-linear} shows that if $0\leq\phi''\pp t\leq\theta,$
then $\sum_{n=1}^{N}\phi\pp{\b a_{n}^{\top}\z+b_{n}}$ is $M$-matrix-smooth
for $M=\theta\sum_{i=1}^{N}\b a_{i}\b a_{i}^{\top}.$ Applying this\footnote{For linear regression, set $\phi\pp t=-t^{2}/\pp{2\rho^{2}}$, $\b a_{n}=\x_{n}$
and $b_{n}=y_{n}$ and observe that $\phi''=-1/\rho^{2}$. For logistic
regression, set $\phi\pp t=\log\mathrm{Sigmoid}\pp t$, $\b a_{n}=y_{n}\x_{n}$
and $b_{n}=0$ and observe that $\phi''\leq1/4$. Adding the prior
and using the triangle inequality gives the result.} gives that $f\pp{\z}$ and $f_{n}\pp{\z}$ are matrix-smooth for 

\[
M=\frac{1}{\sigma^{2}}I+c\sum_{n=1}^{N}\x_{n}\x_{n}^{\top},\text{ and}\ \ \ \ M_{n}=\frac{1}{N\sigma^{2}}I+c\ \x_{n}\x_{n}^{\top},
\]
\begin{wraptable}[11]{o}{0.45\columnwidth}%
\begin{centering}
\vspace{-.35cm}%
\begin{tabular}{cccc}
Dataset & Type & \# data & \# dims\tabularnewline
\hline 
boston & r & 506 & 13\tabularnewline
fires & r & 517 & 12\tabularnewline
cpusmall & r & 8192 & 13\tabularnewline
a1a & c & 1695 & 124\tabularnewline
ionosphere & c & 351 & 35\tabularnewline
australian & c & 690 & 15\tabularnewline
sonar & c & 208 & 61\tabularnewline
mushrooms & c & 8124 & 113\tabularnewline
\end{tabular}
\par\end{centering}
\caption{Regression (r) and classification (c) datasets\label{tab:datasets}}
\end{wraptable}%
where $c=1/\rho^{2}$ for linear regression, and $c=1/4$ for logistic
regression. Taking the spectral norm of these matrices gives scalar
smoothness constants. With subsampling, this is $\VV{M_{n}}_{2}=\frac{1}{\sigma^{2}N}+c\VV{\x_{n}}^{2}$.

\subsection{Evaluation of Bounds}

To enable a clear comparison of of different estimators and bounds,
we generate a single optimization trace of parameter vectors $\w$
for each dataset. All comparisons use this same trace. These use a
conservative optimization method: Find a maximum $\zo$ and then initialize
to $\b w=\pp{\zo,0}$. Then, optimization uses proximal stochastic
gradient descent (with the proximal operator reflecting $h$) with
a step size of $1/M$ (the scalar smoothness constant) and 1000 evaluations
for each gradient estimate.

\ref{fig:traces} shows the evolution of the ELBO along with the variance
of gradient estimation either in batch or stochastically with a uniform
distribution over data. For each iteration and estimator, we plot
the empirical $\VV{\gr}^{2}$ along with this paper's bounds using
either scalar or matrix smoothness.

\begin{figure}[t]
\includegraphics[viewport=0bp 28.35bp 256.5bp 144bp,clip,scale=0.65]{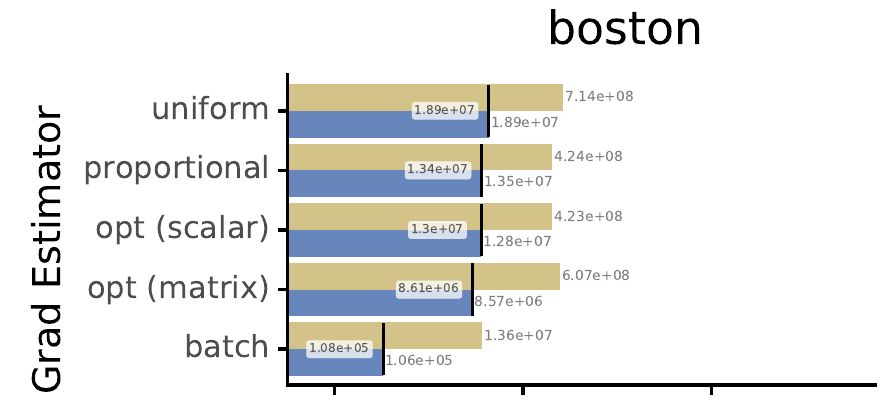}\includegraphics[viewport=81bp 28.35bp 256.5bp 144bp,clip,scale=0.65]{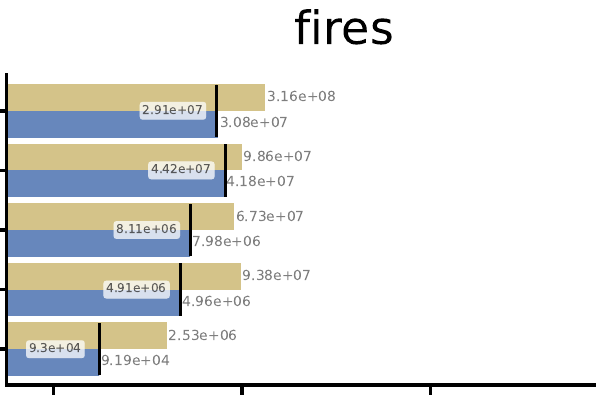}\includegraphics[viewport=81bp 28.35bp 256.5bp 144bp,clip,scale=0.65]{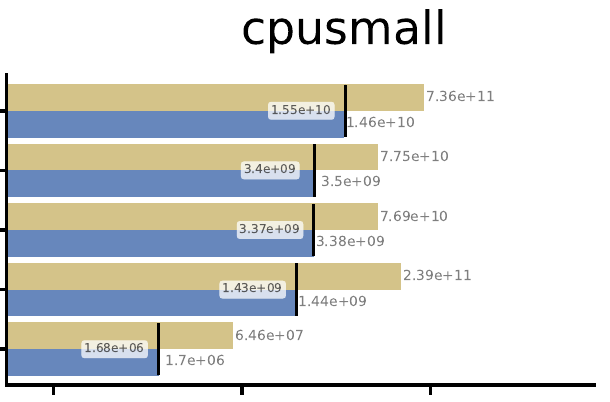}

\includegraphics[viewport=0bp 28.35bp 256.5bp 144bp,clip,scale=0.65]{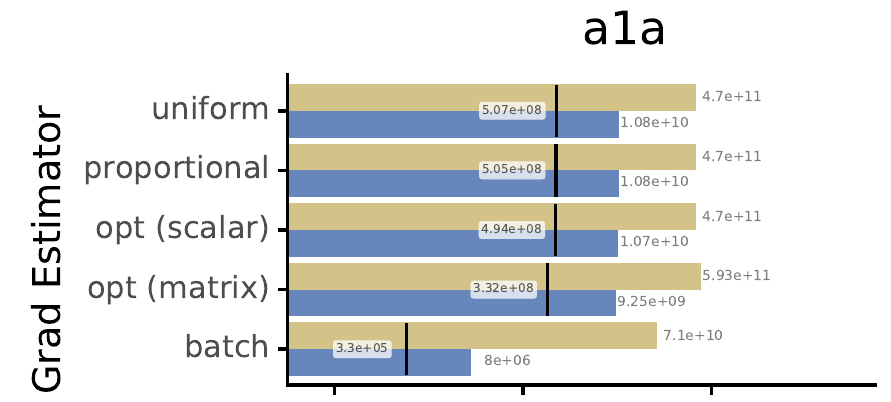}\includegraphics[viewport=81bp 28.35bp 256.5bp 144bp,clip,scale=0.65]{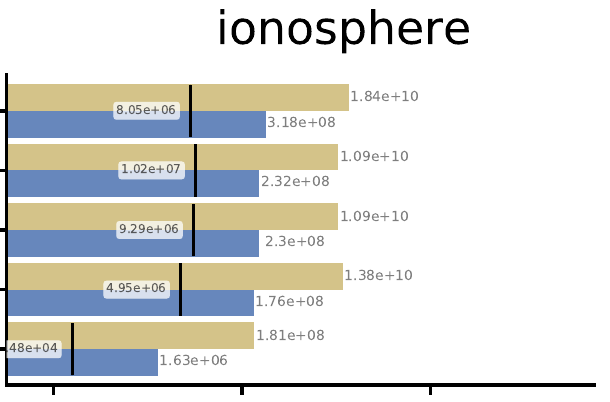}\includegraphics[viewport=81bp 28.35bp 256.5bp 144bp,clip,scale=0.65]{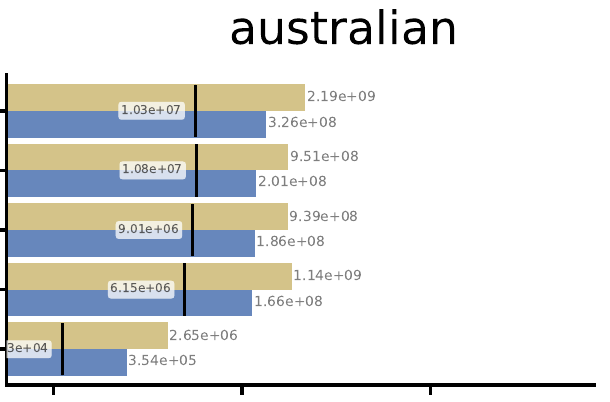}

\includegraphics[viewport=0bp 0bp 256.5bp 144bp,clip,scale=0.65]{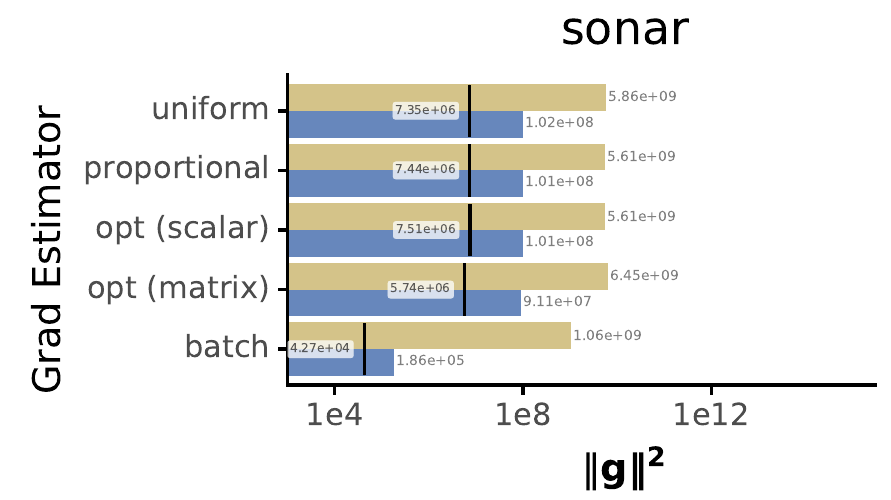}\includegraphics[viewport=81bp 0bp 360bp 144bp,clip,scale=0.65]{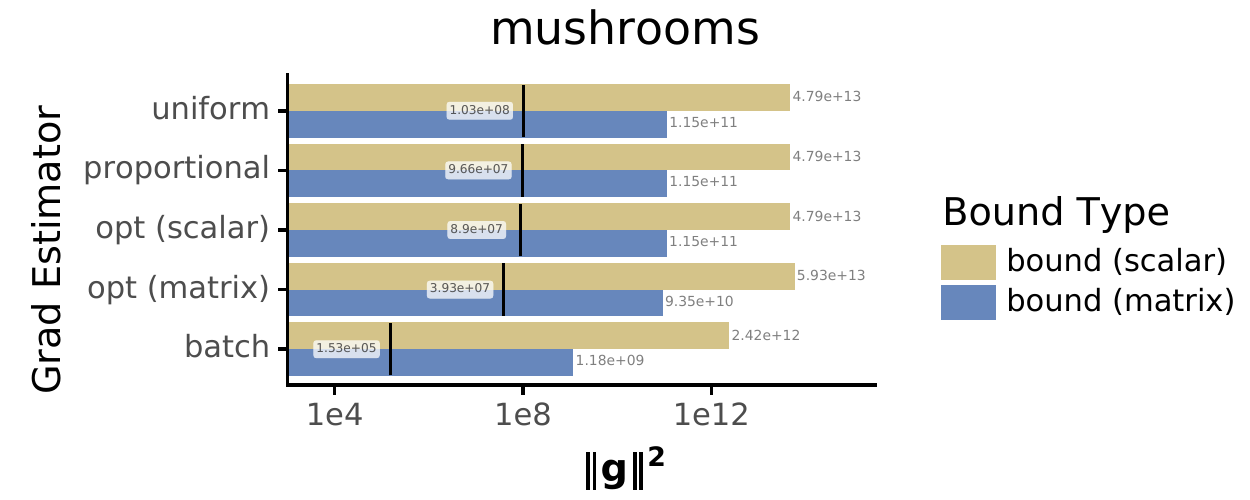}

\caption{\textbf{Tightening variance bounds reduces true variance.} A comparison
of the true (vertical bars) and bounded $\protect\E\protect\Verts{\protect\gr}^{2}$
values produced using five different gradient estimators. \texttt{Batch}
does not use subsampling. \texttt{Uniform} uses subsampling $\pi\protect\pp n=1/N$,
\texttt{proportional} uses $\pi\protect\pp n\propto M_{n}$, \texttt{opt
(scalar)} numerically optimizes $\pi\protect\pp n$ to tighten \ref{eq:g2_stoch}
with a scalar $M_{n}$ and \texttt{opt (matrix)} tightens \ref{eq:g2_stoch}
with a matrix $M_{n}$. For each sampling strategy, we show the variance
bound both with a scalar and matrix $M_{n}$. Uniform sampling has
true and bounded values of $\protect\E\protect\Verts{\protect\br g}^{2}$
ranging between 1.5x and 10x higher than those for sampling with $\pi$
numerically optimized.\label{fig:Tightening}}
\end{figure}

\subsection{Sampling distributions}

With subsampling, variability depends on the sampling distribution
$\pi$. We consider uniform sampling as well as three strategies that
attempt to tighten the bound in \ref{thm:g2-stoch}. In general, $\sum_{n}f\pp n^{2}/\pi\pp n$
is minimized over distributions $\pi$ by $\pi\pp n\propto\verts{f\pp n}$.
Thus, the tightest bound is given by

\begin{equation}
\pi_{\b w}^{*}\pp n\propto\sqrt{\pars{d+1}\Verts{M_{n}\pp{\b m-\zo_{n}}}_{2}^{2}+\pars{d+\kappa}\Verts{M_{n}C}_{F}^{2}}.\label{eq:pi_star}
\end{equation}

We call this ``opt (scalar)'' or ``opt (matrix)'' when $M_{n}$
is a scalar or matrix, respectively. We also consider a ``proportional''
heuristic with $\pi\pp n\propto M_{n}$ for a scalar $M_{n}.$ Sampling
from \ref{eq:pi_star} appears to require calculating the right-hand
side for each $n$ and then normalizing, which may not be practical
for large datasets. While there are obvious heuristics for recursively
approximating $\pi^{*}$ during an optimization, to maintain focus
we do not pursue these ideas here.

\ref{fig:Tightening} shows the empirical and true variance at the
final iteration of the optimization shown in \ref{fig:traces}. The
basic conclusion is that using a more careful sampling distribution
reduces both true and empirical variance.

\section{Discussion\label{sec:Discussion}}

\subsection{Related work\label{subsec:Related-work}}

\citet{Xu_2019_Variancereductionproperties} compute the variance
of a reparameterization estimator applied to a quadratic function,
when the variational distribution is a fully-factorized Gaussian.
This paper can be seen as extending this result to more general densities
(full-rank location-scale families) and more general target functions
(smooth functions).

\citet{Fan_2015_FastSecondOrderStochastic} give an abstract variance
bound for RP estimators. Essentially, they argue that if $\r g_{i}=\nabla_{w_{i}}f\pp{\T_{\w}\pp{\ur}}$
and $\nabla_{w_{i}}f\pp{\T_{\w}\pp{\u}}$ is $M$-smooth as a function
of $\u$, then $\V\bb{\r g_{i}}\leq M^{2}\pi^{2}/4$ when $\ur\sim\N(0,I).$
While this result is fairly abstract -- there is no proof that the
smoothness assumption holds for any particular $M$ with any particular
$f$ and $\T_{\w}$ -- it is similar in spirit to the results in
this paper.

\subsection{Variance vs Expected Squared Norms}

The above results are on the the expected squared norm (ESN) of the
gradient $\E\Vert\gr\Vert^{2}.$ Some stochastic gradient convergence
rates  instead consider (the trace of) the variance $\V\bb{\gr}$.
Since $\tr\V\bb{\gr}=\E\Vert\gr\Vert^{2}-\Vert\E\gr\Vert^{2}$, ESN
bounds are valid as variance bounds. Still, one can ask if these bounds
are loose. The following (proof in \ref{subsec:Variance-vs-Expected})
gives a lower-bound that shows that there is not much to gain from
a direct bound on the variance rather than just using the ESN bound
from \ref{thm:g2-stoch}.

\begin{restatable}{thm1}{varvesn}For any symmetric matrices $M_{1},\cdots,M_{N}$
and vectors $\bar{\z}_{1},\cdots,\bar{\z}_{N}$, there are functions
$f_{1},\cdots,f_{N}$ such that (1) $f_{n}$ is $M_{n}$-matrix-smooth
and has a stationary point at $\bar{\z}_{n}$ and (2) if $s$ is standardized
with $\ur\in\R^{d}$ and $\E\ur_{i}^{4}=\kappa$, then for $\r g=\frac{1}{\pi\pp{\r n}}\nabla f_{\r n}\pp{\T_{\b w}\pp{\ur}},$\label{thm:var-vs-esn}
\[
\tr\V\Vert\gr\Vert_{2}^{2}\geq\sum_{n=1}^{N}\frac{1}{\pi\pp n}\pars{d\Verts{M_{n}\pp{\b m-\zo_{n}}}_{2}^{2}+\pars{d+\kappa-1}\Verts{M_{n}C}_{F}^{2}}.
\]
\end{restatable}

When $d\gg1$ this lower-bound is very close to the upper-bound on
$\E\Verts{\gr}^{2}$ in \ref{thm:g2-stoch}. Thus, under this paper's
assumptions, a variance bound cannot be significantly better than
an ESN bound.

\subsection{The Entropy Term\label{subsec:The-Entropy-Term}}

All discussion in this paper has been for gradient estimators for
$l,$ while the goal is of course to optimize $l+h.$ For location-scale
families, $h$ is known in closed-form, meaning the exact gradient
-- or the proximal operator for $h$ -- can be computed exactly.
Still, it has been observed that if $q_{\w}$ is very close to $p\pp{\z|\x},$
cancellations mean that estimating the gradient of $h+l$ might have
lower variance than the gradient of $l$ alone \citep{Roeder_2017_StickingLandingSimple}.

With any variational family, it is well-known that the gradient of
the entropy can be represented as $-\nabla_{\w}\E_{\zr\sim q_{\w}}\log q_{\v}\pp{\zr}\vert_{\v=\w}.$
That is, the dependence of $\log q_{\w}$ on $\w$ can be neglected
under differentiation. Thus, if one wishes to stochastically estimate
the gradient of $h$, one can treat $\log q_{\v}$ in the same way
as $\log p$ when calculating gradients. Then, one could apply the
analysis in this paper to $f\pp{\z}=\log p\pp{\z,\x}-\log q_{\v}\pp{\z}$
rather than $f\pp{\z}=\log p\pp{\z,\x}$ as done above. It is easy
to imagine situations where subtracting $\log q_{\v}$ (or a fraction
of it) from $\log p$ would change $M_{n}$ and $\bar{\z}_{n}$ in
such a way as to produce a tighter bound. Thus, the bounds in this
paper are consistent with practices \citep{Geffner_2018_UsingLargeEnsembles,Roeder_2017_StickingLandingSimple}
where using $\log q_{\v}$ as a control variate can reduce gradient
variance.

\subsection{Smoothness and Convergence Guarantees}

At a very high level, convergence rates for stochastic gradient methods
require both (1) control of the variability of the gradient estimator
and (2) either convexity or Lipschitz smoothness of the objective.
This paper is dedicated entirely to the first goal. Independent recent
work has addressed at the second issue \citep{Domke_2019_ProvableSmoothnessGuarantees}.
The basic summary is that if $f\pp{\z}$ is smooth, then $l\pp{\w}$
is smooth, and similarly if $f\pp{\z}$ is strongly convex. However,
full convergence guarantees for black-box VI remain an open research
problem.

\subsection{Prospects for Generalizing Bounds to Other Variational Families}

The bounds given in this paper are closely tied to location-scale
families: The exact form of the reparameterization function $\T_{w}$
is used in \ref{lem:gradnorm} and \ref{lem:expectedtminuszstar},
which underly the main results of \ref{thm:batch_f_scalar_M}, \ref{thm:batch_f_matrix_M},
and \ref{eq:thm2_result}. Thus, extending our proof strategy to other
variational families would require deriving new results analogous
to \ref{lem:gradnorm} and \ref{lem:expectedtminuszstar} for the
reparameterization function $\T_{w}$ corresponding to those new variational
families. Moreover, if the exact entropy is not available for a variational
family, the analysis must address the variance of the entropy gradient
estimator, as discussed in \ref{subsec:The-Entropy-Term}.

\subsection{Limitations}

This work has several limitations. First, it applies only to location-scale
families, and requires that the target objective be smooth. Second,
if $\log p$ is smooth, it may still be challenging in practice to
establish what the smoothness constant is. Third, we observed that
even with our strongest condition of matrix smoothness, the some looseness
remains in the bounds with the logistic regression examples. Since
the ESN bound is unimprovable, this looseness cannot be removed without
using more detailed structure of the target $\log p$. It is not obvious
what this structure would be, or how it would be obtained for practical
black-box inference problems.

\bibliographystyle{plainnat}
\bibliography{/Users/domke/Dropbox/Papers/Bibliography/justindomke_zotero_betterbibtex2}

\clearpage{}

\newpage{}

\section{Additional Experimental Results}

\begin{figure}[t]
\includegraphics[viewport=0bp 25.2bp 252bp 144bp,clip,scale=0.8]{\string"Variance Tests/boston_elbo\string".pdf}\includegraphics[viewport=0bp 25.2bp 252bp 144bp,clip,scale=0.8]{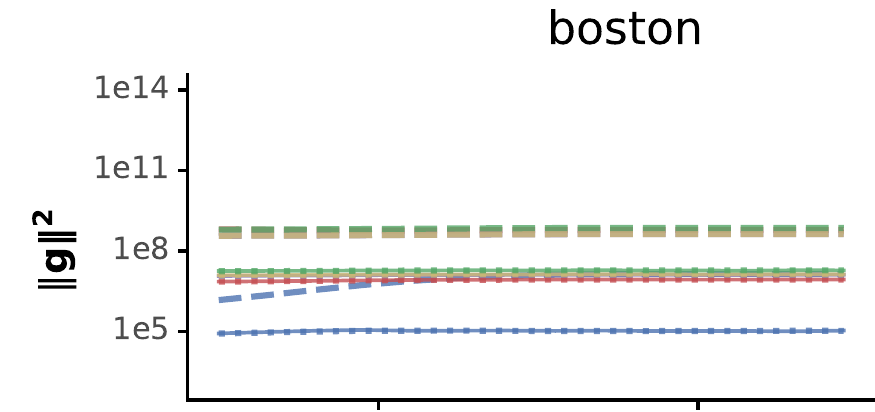}

\includegraphics[viewport=0bp 25.2bp 252bp 144bp,clip,scale=0.8]{\string"Variance Tests/fires_elbo\string".pdf}\includegraphics[viewport=0bp 25.2bp 252bp 144bp,clip,scale=0.8]{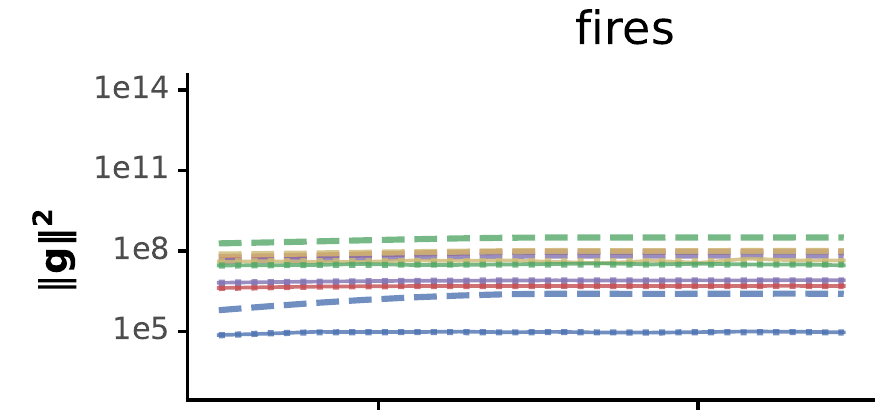}

\includegraphics[viewport=0bp 0bp 252bp 144bp,clip,scale=0.8]{\string"Variance Tests/cpusmall_elbo\string".pdf}\includegraphics[viewport=0bp 0bp 360bp 144bp,clip,scale=0.8]{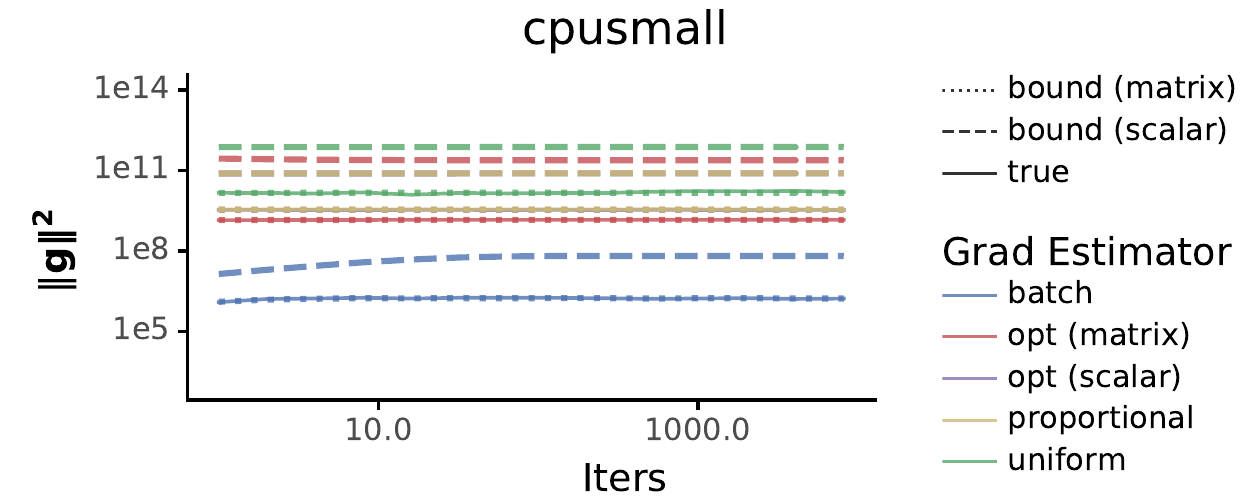}

\caption{More results in the same setting as \ref{fig:traces} (regression
data)}
\end{figure}

\begin{figure}[t]
\includegraphics[viewport=0bp 25.2bp 252bp 144bp,clip,scale=0.8]{\string"Variance Tests/a1a_elbo\string".pdf}\includegraphics[viewport=0bp 25.2bp 252bp 144bp,clip,scale=0.8]{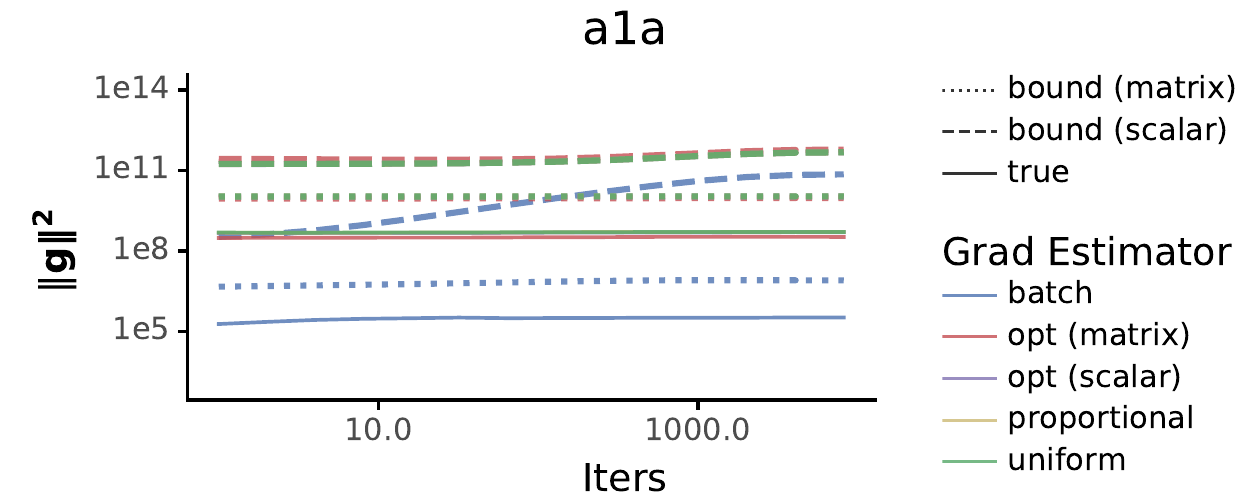}

\includegraphics[viewport=0bp 25.2bp 252bp 144bp,clip,scale=0.8]{\string"Variance Tests/ionosphere_elbo\string".pdf}\includegraphics[viewport=0bp 25.2bp 252bp 144bp,clip,scale=0.8]{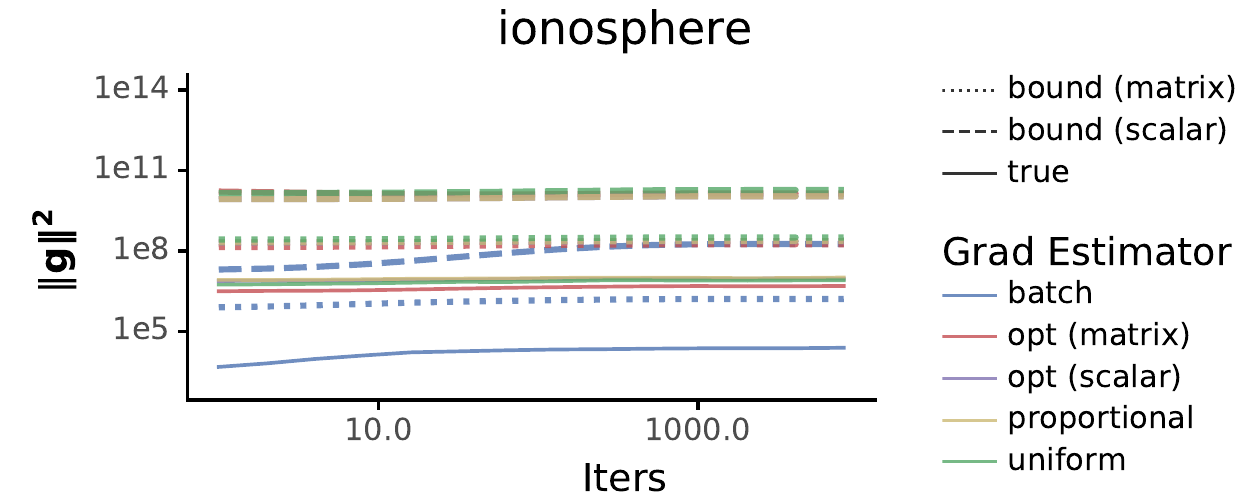}

\includegraphics[viewport=0bp 25.2bp 252bp 144bp,clip,scale=0.8]{\string"Variance Tests/australian_elbo\string".pdf}\includegraphics[viewport=0bp 25.2bp 252bp 144bp,clip,scale=0.8]{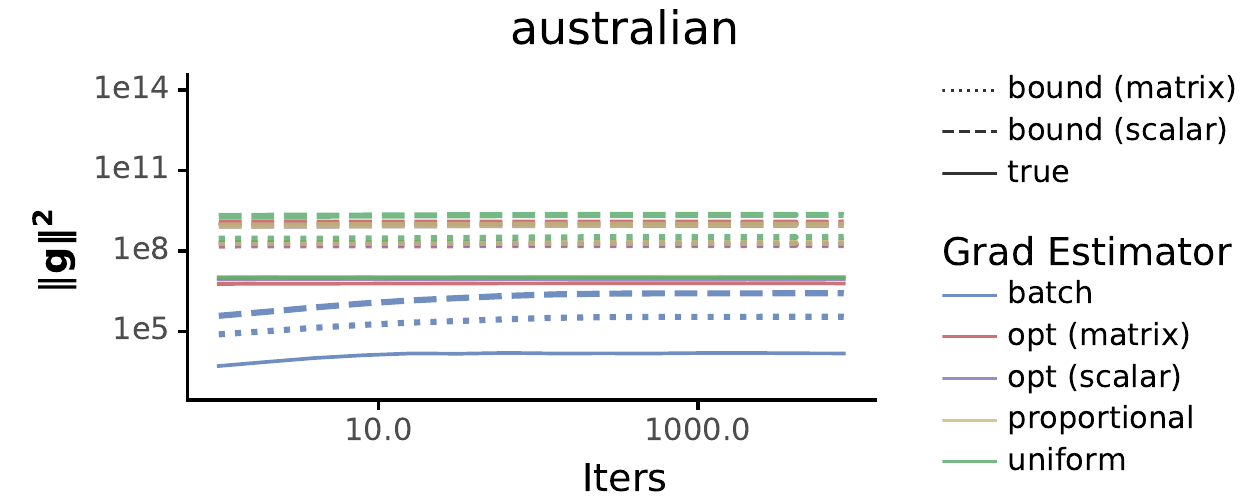}

\includegraphics[viewport=0bp 25.2bp 252bp 144bp,clip,scale=0.8]{\string"Variance Tests/sonar_elbo\string".pdf}\includegraphics[viewport=0bp 25.2bp 252bp 144bp,clip,scale=0.8]{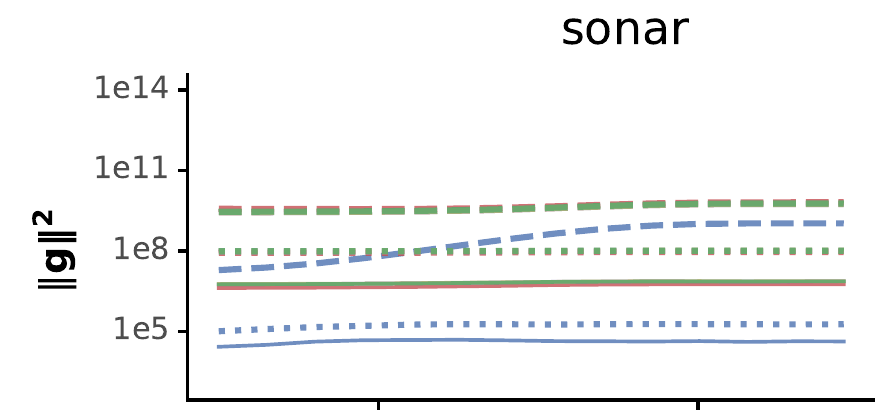}

\includegraphics[viewport=0bp 0bp 252bp 144bp,clip,scale=0.8]{\string"Variance Tests/mushrooms_elbo\string".pdf}\includegraphics[viewport=0bp 0bp 360bp 144bp,clip,scale=0.8]{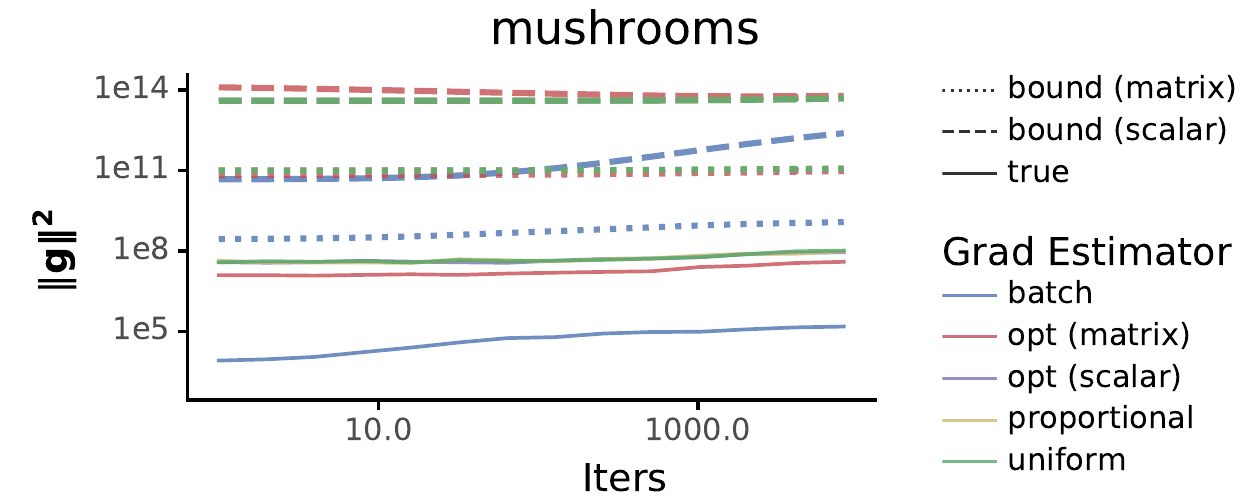}

\caption{More results in the same setting as \ref{fig:traces} (classification
data)}
\end{figure}

\clearpage{}

\section{Proofs\label{sec:Proofs}}

\subsection{Proof of \ref{lem:gradnorm}}

The following result is helpful for establishing \ref{lem:gradnorm}.

\begin{restatable}{lem1}{jactTjact}If $\nabla_{\b w}\b t_{\b w}\pp{\u}$
is Jacobian-transpose of $\b t_{\b w}\pp{\u}$ with respect to $\b w$,
then\label{lem:jac-t-jac-t}
\[
\nabla_{\b w}\T_{\b w}\pp{\u}^{\top}\nabla_{\b w}\T_{\b w}\pp{\u}=I\pp{1+\Verts{\u}_{2}^{2}}.
\]
\end{restatable}
\begin{proof}
We use the notation $\nabla_{\b w}\T_{\b w}\pp{\u}=\frac{d\T_{\b w}\pp{\u}^{\top}}{d\b w}$,
meaning that $\pars{\nabla_{\b w}\T_{\b w}\pp{\u}}_{ij}=\frac{d\T_{\b w}\pp{\u}_{j}}{dw_{i}}.$

Each row of $\nabla_{\b w}\b t_{\b w}\pp{\u}$ consists of the partial
derivative of $\b t_{\b w}\pp{\u}$ with respect to one component
of $\b w$. Thus, the product is
\begin{eqnarray*}
\pars{\nabla_{\b w}\T_{\b w}\pp{\u}}^{\top}\pars{\nabla_{\b w}\T_{\b w}\pp{\u}} & = & \sum_{i}\pars{\frac{d}{dw_{i}}\T\pp{\u}}\pars{\frac{d}{dw_{i}}\T\pp{\u}}^{\top}\\
 & = & \sum_{i}\b a_{i}\pp{\u}\b a_{i}\pp{\u}^{\top}.
\end{eqnarray*}
We can calculate these components as

\begin{eqnarray*}
\pars{\frac{d}{dm_{i}}\T_{\b w}\pp{\u}}\pars{\frac{d}{dm_{i}}\T_{\b w}\pp{\u}}^{\top} & = & \b e_{i}\b e_{i}^{\top}\\
\pars{\frac{d}{dS_{ij}}\T_{\b w}\pp{\u}}\pars{\frac{d}{dS_{ij}}\T_{\b w}\pp{\u}}^{\top} & = & \pars{u_{j}\b e_{i}}\pars{u_{j}\b e_{i}}^{\top}\\
 & = & u_{j}^{2}\b e_{i}\b e_{i}^{\top}
\end{eqnarray*}
Adding the components up, we get that
\begin{eqnarray*}
\pars{\nabla_{\b w}\T_{\b w}\pp{\u}}^{\top}\pars{\nabla_{\b w}\T_{\b w}\pp{\u}} & = & \sum_{i}\pars{\frac{d}{dm_{i}}\T_{\b w}\pp{\u}}\pars{\frac{d}{dm_{i}}\T_{\b w}\pp{\u}}^{\top}+\sum_{i,j}\pars{\frac{d}{dS_{ij}}\T_{\b w}\pp{\u}}\pars{\frac{d}{dS_{ij}}\T_{\b w}\pp{\u}}^{\top}\\
 & = & \sum_{i}\b e_{i}\b e_{i}^{\top}+\sum_{i,j}u_{j}^{2}\b e_{i}\b e_{i}^{\top}\\
 & = & I\pp{1+\Verts{\u}_{2}^{2}}.
\end{eqnarray*}

The following is the main Lemma.
\end{proof}
\gradnorm*
\begin{proof}
Using Lemma \ref{lem:jac-t-jac-t}, we can show that
\begin{eqnarray*}
\Verts{\nabla_{\b w}f\pp{\T_{\b w}\pp{\u}}}_{2}^{2} & = & \Verts{\nabla_{\b w}\T_{\b w}\pp{\u}\ \nabla f\pp{\T_{\b w}\pp{\u}}}_{2}^{2}\\
 & = & \nabla f\pp{\T_{\b w}\pp{\u}}^{\top}\nabla_{\b w}\T_{\b w}\pp{\u}^{\top}\nabla_{\b w}\T_{\b w}\pp{\u}\ \nabla f\pp{\T_{\b w}\pp{\u}}\\
 & = & \nabla f\pp{\T_{\b w}\pp{\u}}^{\top}\pars{I\pp{1+\Verts{\u}_{2}^{2}}}\ \nabla f\pp{\T_{\b w}\pp{\u}}\\
 & = & \Verts{\nabla f\pp{\T_{\b w}\pp{\u}}}_{2}^{2}\pars{1+\Verts{\u}_{2}^{2}}.
\end{eqnarray*}

\end{proof}
\clearpage{}

\subsection{Proof of \ref{lem:expectedtminuszstar}}

A few distributional properties are needed before proving \ref{lem:expectedtminuszstar}.

\begin{restatable}{lem1}{uexpectations}Suppose that $\ur=\pp{\ur_{1},\cdots,\ur_{d}}$
is random variable over $\R^{d}$ with zero-mean iid components. Then
\begin{eqnarray*}
\E\ur\ur^{\top} & = & \E\bb{\ur_{1}^{2}}I\\
\E\Verts{\ur}_{2}^{2} & = & d\E\bb{\ur_{1}^{2}}\\
\E\ur\pp{1+\Verts{\ur}_{2}^{2}} & = & \b 1\ \E\bb{\ur_{1}^{3}}\\
\E\ur\ur^{\top}\ur\ur^{\top} & = & \pars{\pp{d-1}\E\bb{\ur_{1}^{2}}^{2}+\E\bb{\ur_{1}^{4}}}I.
\end{eqnarray*}
 \end{restatable}
\begin{proof}
($\E\ur\ur^{\top}$) Take any pair of indices $i$ and $j$. Then,
$\pars{\E\ur\ur^{\top}}_{ij}=\E\ur_{i}\ur_{j}.$ If $i\not=j$ this
is zero. Otherwise it is $\E\ur_{1}^{2}.$ Thus, $\E\ur\ur^{\top}=\E\bb{\ur_{1}^{2}}I.$

($\E\Verts{\ur}_{2}^{2}$) This follows from the previous result as
\[
\E\Verts{\ur}_{2}^{2}=\E\tr\ur\ur^{\top}=\tr\E\ur\ur^{\top}=\tr\E\bb{\ur_{1}^{2}}I=d\E\bb{\ur_{1}^{2}}.
\]
($\E\ur\pp{1+\Verts{\ur}_{2}^{2}}$) If $\r x$ and $\r y$ are independent,
$\E\r x\r y=\pp{\E\r x}\pp{\E\r y}.$ Thus, since the first and third
moments of $\ur_{i}$ are zero,
\begin{eqnarray*}
\E\ur\pp{1+\Verts{\ur}_{2}^{2}}_{i} & = & \E\ur_{i}\pp{1+\sum_{j=1}^{d}\ur_{j}^{2}}\\
 & = & \E\bb{\ur_{i}}+\E\bb{\ur_{i}^{3}}+\sum_{j\not=i}\E\bb{\ur_{i}}\E\bb{\ur_{j}^{2}}\\
 & = & \E\bb{\ur_{i}^{3}}.
\end{eqnarray*}

($\E\ur\ur^{\top}\ur\ur^{\top}$) It is useful to represent this term
as
\begin{eqnarray*}
\pars{\E\ur\ur^{\top}\ur\ur^{\top}}_{ij} & = & \E\ur_{i}\ur_{j}\Verts{\ur}_{2}^{2}\\
 & = & \E\ur_{i}\ur_{j}\sum_{k}\ur_{k}^{2}.
\end{eqnarray*}

First, suppose that $i\not=j$. Then this is
\begin{eqnarray*}
\pars{\E\ur\ur^{\top}\ur\ur^{\top}}_{ij} & = & \E\ur_{i}\ur_{j}\sum_{k}\ur_{k}^{2}\\
 & = & \E\ur_{i}\ur_{j}\pars{\ur_{i}^{2}+\ur_{j}^{2}+\sum_{k\not\in\{i,j\}}\ur_{k}^{2}}.\\
 & = & 0.
\end{eqnarray*}
This is zero since $\ur_{i}$, $\ur_{j}$ and $\ur_{k}$ are independent,
and each term contains at least one of $\ur_{i}$ or $\ur_{j}$ to
the first power. Since $\E\ur_{i}=0,$ the full expectation is zero.

On the other hand, suppose that $i=j.$ Then this is
\begin{eqnarray*}
\pars{\E\ur\ur^{\top}\ur\ur^{\top}}_{ii} & = & \E\ur_{i}^{2}\pars{\ur_{i}^{2}+\sum_{k\not=i}\ur_{k}^{2}}\\
 & = & \E\pars{\ur_{i}^{4}+\ur_{i}^{2}\sum_{k\not=i}\ur_{k}^{2}}\\
 & = & \E\bb{\ur_{1}^{4}}+\pp{d-1}\E\bb{\ur_{1}^{2}}^{2}
\end{eqnarray*}

If we put this together, we get that
\[
\E\ur\ur^{\top}\ur\ur^{\top}=\pars{\pp{d-1}\E\bb{\ur_{1}^{2}}^{2}+\E\bb{\ur_{1}^{4}}}I.
\]
\end{proof}
\expectedtminuszstar*
\begin{proof}
We simply split the expectation up and calculate each part.

\begin{eqnarray*}
\E\Verts{\T_{\b w}\pp{\ur}-\zo}_{2}^{2}\pars{1+\Verts{\ur}_{2}^{2}} & = & \E\Verts{C\ur+\b m-\zo}_{2}^{2}\pars{1+\Verts{\ur}_{2}^{2}}\\
 & = & \E\pars{\Verts{C\ur}_{2}^{2}+2\pp{\b m-\zo}^{\top}C\ur+\Verts{\b m-\zo}_{2}^{2}}\pars{1+\Verts{\ur}_{2}^{2}}\\
\E\Verts{C\ur}_{2}^{2}\pars{1+\Verts{\ur}_{2}^{2}} & = & \E\Verts{C\ur}_{2}^{2}+\E\Verts{C\ur}_{2}^{2}\Verts{\ur}_{2}^{2}\\
\E\Verts{C\ur}_{2}^{2} & = & \E\tr\ur^{\top}C^{\top}C\ur\\
 & = & \tr C^{\top}C\E\ur\ur^{\top}\\
 & = & \tr C^{\top}C\E\bb{\ur_{1}^{2}}I\\
 & = & \E\bb{\ur_{1}^{2}}\ \tr C^{\top}C\\
\E\Verts{C\ur}_{2}^{2}\Verts{\ur}_{2}^{2} & = & \E\tr\ur^{\top}C^{\top}C\ur\ur^{\top}\ur\\
 & = & \tr C^{\top}C\ \E\ur\ur^{\top}\ur\ur^{\top}\\
 & = & \tr C^{\top}C\pars{\pp{d-1}\E\bb{\ur_{1}^{2}}^{2}+\E\bb{\ur_{1}^{4}}}I\\
 & = & \pars{\pp{d-1}\E\bb{\ur_{1}^{2}}^{2}+\E\bb{\ur_{1}^{4}}}\tr C^{\top}C\\
\E\Verts{C\ur}_{2}^{2}\pars{1+\Verts{\ur}_{2}^{2}} & = & \pars{\E\bb{\ur_{1}^{2}}+\pp{d-1}\E\bb{\ur_{1}^{2}}^{2}+\E\bb{\ur_{1}^{4}}}\tr C^{\top}C\\
\E\pp{\b m-\zo}^{\top}C\u\pars{1+\Verts{\ur}_{2}^{2}} & = & \pp{\b m-\zo}^{\top}C\ \E\u\pars{1+\Verts{\ur}_{2}^{2}}\\
 & = & \pp{\b m-\zo}^{\top}C\ \b 1\ \E\bb{\ur_{1}^{3}}\\
 & = & 0\\
\E\Verts{\b m-\zo}_{2}^{2}\pp{1+\Verts{\ur}_{2}^{2}} & = & \Verts{\b m-\zo}_{2}^{2}\E\pp{1+\Verts{\ur}_{2}^{2}}\\
 & = & \Verts{\b m-\zo}_{2}^{2}\pp{1+d\E\bb{\ur_{1}^{2}}}.
\end{eqnarray*}
Adding all this up gives that
\[
\E\Verts{\T_{\b w}\pp{\ur}-\zo}_{2}^{2}\pars{1+\Verts{\ur}_{2}^{2}}=\pars{1+d\E\bb{\ur_{1}^{2}}}\Verts{\b m-\zo}_{2}^{2}+\pars{\E\bb{\ur_{1}^{2}}+\pp{d-1}\E\bb{\ur_{1}^{2}}^{2}+\E\bb{\ur_{1}^{4}}}\Verts C_{F}^{2}.
\]
In the case that the variance is one, this becomes
\[
\E\Verts{\T_{\b w}\pp{\ur}-\zo}_{2}^{2}\pars{1+\Verts{\ur}_{2}^{2}}=\pars{d+1}\Verts{\b m-\zo}_{2}^{2}+\pars{d+\E\bb{\ur_{1}^{4}}}\Verts C_{F}^{2}.
\]
\end{proof}
\clearpage{}

\subsection{Proof of \ref{thm:var-vs-esn}\label{subsec:Variance-vs-Expected}}

\varvesn*
\begin{proof}
First, take any matrix $M$ and vector $\bar{\b z}.$ Define
\[
f\pp{\z}=\frac{1}{2}\pp{\z-\bar{\z}}^{\top}M\pp{\z-\bar{\z}}.
\]
We can calculate that
\begin{eqnarray*}
l\pp{\w} & = & \E_{\zr\sim q_{\w}}\frac{1}{2}\pp{\z-\bar{\z}}^{\top}M\pp{\z-\bar{\z}}\\
 & = & \E_{\zr\sim q_{\w}}\frac{1}{2}\z^{\top}M\z-\E_{\zr\sim q_{\w}}\bar{\z}^{\top}M\z+\E_{\zr\sim q_{\w}}\frac{1}{2}\bar{\z}^{\top}M\bar{\z}\\
 & = & \E_{\zr\sim q_{\w}}\frac{1}{2}\tr M\z\z^{\top}-\bar{\z}^{\top}M\b m+\frac{1}{2}\bar{\z}^{\top}M\bar{\z}\\
 & = & \frac{1}{2}\tr M\pp{\b m\b m^{\top}+CC^{\top}}-\bar{\z}^{\top}M\b m+\frac{1}{2}\bar{\z}^{\top}M\bar{\z}\\
 & = & \frac{1}{2}\b m^{\top}M\b m+\frac{1}{2}\tr MCC^{\top}-\bar{\z}^{\top}M\b m+\frac{1}{2}\bar{\z}^{\top}M\bar{\z}\\
 & = & \frac{1}{2}\pp{\b m-\bar{\z}}^{\top}M\pp{\b m-\bar{\z}}+\frac{1}{2}\tr MCC^{\top}.
\end{eqnarray*}
Thus, we have that
\begin{eqnarray*}
\frac{dl}{d\b m} & = & M\pp{\b m-\bar{\z}}\\
\frac{dl}{dC} & = & MC
\end{eqnarray*}
If we add up components, we get that
\[
\VV{\E\gr}_{2}^{2}=\VV{\nabla l\pp{\w}}_{2}^{2}=\VV{M\pp{\b m-\bar{\z}}}_{2}^{2}+\VV{MC}_{F}^{2}.
\]

Now, given a sequence $M_{1},\cdots,M_{N}$ and $\bar{\z}_{1},\cdots,\bar{\z}_{N}$,
if we choose
\[
f_{n}\pp{\z}=\frac{1}{2}\pp{\z-\bar{\z}_{n}}^{\top}M_{n}\pp{\z-\bar{\z}_{n}},
\]
The true gradient will be
\begin{eqnarray*}
\frac{dl}{d\b m} & = & \sum_{n=1}^{N}M_{n}\pp{\b m-\bar{\z}_{n}}\\
\frac{dl}{dC} & = & M_{n}C,
\end{eqnarray*}
and so, applying Jensen's inequality,
\begin{eqnarray*}
\Verts{\E\r g}_{2}^{2} & = & \Verts{\nabla l\pp{\w}}_{2}^{2}\\
 & = & \Verts{\sum_{n=1}^{N}M_{n}\pp{\b m-\bar{\z}_{n}}}_{2}^{2}+\Verts{\sum_{n=1}^{N}M_{n}C}_{F}^{2}\\
 & = & \Verts{\sum_{n=1}^{N}\frac{1}{\pi\pp n}\pi\pp nM_{n}\pp{\b m-\bar{\z}_{n}}}_{2}^{2}+\Verts{\sum_{n=1}^{N}\frac{1}{\pi\pp n}\pi\pp nM_{n}C}_{F}^{2}\\
 & \leq & \sum_{n=1}^{N}\pi\pp n\Verts{\frac{1}{\pi\pp n}M_{n}\pp{\b m-\bar{\z}_{n}}}_{2}^{2}+\sum_{n=1}^{N}\pi\pp n\Verts{\frac{1}{\pi\pp n}M_{n}C}_{F}^{2}\\
 & = & \sum_{n=1}^{N}\frac{1}{\pi\pp n}\pars{\Verts{M_{n}\pp{\b m-\bar{\z}_{n}}}_{2}^{2}+\Verts{M_{n}C}_{F}^{2}}.
\end{eqnarray*}
\ref{thm:g2-stoch} tells us that
\[
\E\Verts{\r g}_{2}^{2}=\sum_{n=1}^{N}\frac{1}{\pi\pp n}\pars{\pars{d+1}\Verts{M_{n}\pp{\b m-\zo_{n}}}_{2}^{2}+\pars{d+\kappa}\Verts{M_{n}C}_{F}^{2}}.
\]
Thus, we have that
\begin{eqnarray*}
\tr\V\Vert\gr\Vert_{2}^{2} & = & \E\Vert\gr\Vert^{2}-\Vert\E\gr\Vert^{2}\\
 & \geq & \sum_{n=1}^{N}\frac{1}{\pi\pp n}\pars{d\Verts{M_{n}\pp{\b m-\zo_{n}}}_{2}^{2}+\pars{d+\kappa-1}\Verts{M_{n}C}_{F}^{2}}.
\end{eqnarray*}
\end{proof}

\clearpage{}

\section{Smoothness conditions for linear models\label{sec:Smoothness-conditions-for-linear}}
\begin{lem}
Suppose that $f\pp z=\phi\pp{a^{\top}z},$ and that $\verts{\phi''\pp t}\leq\theta$
for all $t$. Then,
\[
\Verts{\nabla f\pp y-\nabla f\pp z}_{2}\leq\theta\Verts a_{2}\verts{a^{\top}\pp{y-z}}.
\]
\end{lem}

\begin{proof}
Then, we have that
\begin{eqnarray*}
\Verts{\nabla f\pp y-\nabla f\pp z}_{2} & = & \Verts{a\phi'\pp{a^{\top}y}-a\phi'\pp{a^{\top}z}}_{2}\\
 & = & \Verts a_{2}\verts{\phi'\pp{a^{\top}y}-\phi'\pp{a^{\top}z}}\\
 & = & \Verts a_{2}\verts{\int_{a^{\top}z}^{a^{\top}y}\phi''\pp tdt}\\
 & \leq & \theta\Verts a_{2}\verts{a^{\top}\pp{y-z}}.
\end{eqnarray*}
\end{proof}
\begin{lem}
Suppose that $f\pp z=f_{0}\pp z+\phi\pp{a^{\top}z}$ and that $f_{0}\pp z$
is $M_{0}$ smooth. Then, we have that
\begin{eqnarray*}
\Verts{\nabla f\pp y-\nabla f\pp z}_{2} & = & M_{0}\Verts{y-z}_{2}+\theta\Verts a_{2}\verts{a^{\top}\pp{y-z}}.
\end{eqnarray*}
\end{lem}

\begin{lem}
Suppose that $f\pp z=\sum_{i=1}^{N}\phi\pp{a_{i}^{\top}z}$ and that
$0\leq\phi''\pp t\leq\theta$ for all $t$. Then, 
\begin{eqnarray*}
\Verts{\nabla f\pp y-\nabla f\pp z}_{2} & \leq & \Verts{M\pp{y-z}}_{2}\\
M & = & \theta\sum_{i=1}^{N}a_{i}a_{i}^{\top}
\end{eqnarray*}
\end{lem}

\begin{proof}
\begin{eqnarray*}
\Verts{\nabla f\pp y-\nabla f\pp z}_{2} & = & \Verts{\sum_{i=1}^{N}a_{i}\phi'\pp{a_{i}y}-\sum_{i=1}^{N}a_{i}\phi'\pp{a_{i}z}}_{2}\\
 & = & \Verts{\sum_{i=1}^{N}a_{i}\pars{\phi'\pp{a_{i}y}-\phi'\pp{a_{i}z}}}_{2}\\
 & = & \Verts{\sum_{i=1}^{N}a_{i}\int_{a_{i}^{\top}z}^{a_{i}^{\top}y}\phi''\pp tdt}_{2}\\
 & = & \Verts{\sum_{i=1}^{N}a_{i}\pars{a_{i}^{\top}y-a_{i}^{\top}z}b_{i}}_{2}\\
 &  & -\theta\leq b_{i}\leq\theta\\
 & = & \Verts{\sum_{i=1}^{N}b_{i}a_{i}a_{i}^{\top}\pars{y-z}}_{2}\\
 & \leq & \theta\Verts{\pars{\sum_{i=1}^{N}a_{i}a_{i}^{\top}}\pars{y-z}}_{2}
\end{eqnarray*}

The final inequality is justified by the following claim: $\Verts{\sum_{i=1}^{N}b_{i}a_{i}a_{i}^{\top}\pars{y-z}}_{2}^{2}$
is maximized over vectors $b$ with $0\leq b_{i}\leq\theta$ by setting
$b_{i}=\theta$ always. To establish this claim observe that
\begin{eqnarray*}
\frac{d}{db_{k}}\Verts{\sum_{i=1}^{N}b_{i}a_{i}a_{i}^{\top}\pars{y-z}}_{2}^{2} & = & \frac{d}{db_{k}}\pars{\sum_{i=1}^{N}b_{i}a_{i}a_{i}^{\top}\pars{y-z}}^{\top}\pars{\sum_{j=1}^{N}b_{j}a_{j}a_{j}^{\top}\pars{y-z}}\\
 & = & \frac{d}{db_{k}}\sum_{i=1}^{N}\sum_{j=1}^{N}b_{i}b_{j}\pars{y-z}^{\top}\pars{a_{i}a_{i}^{\top}a_{j}a_{j}^{\top}}\pars{y-z}\\
 & = & \frac{d}{db_{k}}2\sum_{j\not=k}^{N}b_{k}b_{j}\pars{y-z}^{\top}\pars{a_{k}a_{k}^{\top}a_{j}a_{j}^{\top}}\pars{y-z}\\
 &  & +\frac{d}{db_{k}}b_{k}^{2}\pars{y-z}^{\top}\pars{a_{k}a_{k}^{\top}a_{k}a_{k}^{\top}}\pars{y-z}\\
 & = & 2\sum_{j\not=k}^{N}b_{j}\pars{y-z}^{\top}\pars{a_{k}a_{k}^{\top}a_{j}a_{j}^{\top}}\pars{y-z}\\
 &  & +2b_{k}\pars{y-z}^{\top}\pars{a_{k}a_{k}^{\top}a_{k}a_{k}^{\top}}\pars{y-z}\\
 & = & 2\sum_{j=1}^{N}b_{j}\pars{y-z}^{\top}\pars{a_{k}a_{k}^{\top}a_{j}a_{j}^{\top}}\pars{y-z}\\
 & = & 2\sum_{j=1}^{N}b_{j}\tr\pars{y-z}^{\top}\pars{a_{k}a_{k}^{\top}a_{j}a_{j}^{\top}}\pars{y-z}\\
 & = & 2\tr a_{k}a_{k}^{\top}\pars{\sum_{j=1}^{N}b_{j}a_{j}a_{j}^{\top}}\pars{y-z}\pars{y-z}^{\top}\\
 & = & 2a_{k}^{\top}\pars{\sum_{j=1}^{N}b_{j}a_{j}a_{j}^{\top}}\pars{y-z}\pars{y-z}^{\top}a_{k}
\end{eqnarray*}
Now, both $\pars{\sum_{j=1}^{N}b_{j}a_{j}a_{j}^{\top}}$ and $\pars{y-z}\pars{y-z}^{\top}$
are real symmetric positive definite matrices. Thus, their product
has real non-negative eigenvalues. This means that
\[
\frac{d}{db_{k}}\Verts{\sum_{i=1}^{N}b_{i}a_{i}a_{i}^{\top}\pars{y-z}}_{2}^{2}\geq0,
\]
i.e. the maximizing $b$ will set all entries to $\theta$.
\end{proof}
\begin{thm}
Suppose that $f\pp z=\frac{c}{2}\Verts z_{2}^{2}+\sum_{i=1}^{N}\phi\pp{a_{i}^{\top}z}$
and that $0\leq\phi''\pp t\leq\theta.$ Then,
\begin{eqnarray*}
\Verts{\nabla f\pp y-\nabla f\pp z}_{2} & \leq & \Verts{M\pp{y-z}}_{2}\\
M & = & cI+\theta\sum_{i=1}^{N}a_{i}a_{i}^{\top}
\end{eqnarray*}
\end{thm}

\begin{proof}
Suppose that $\nabla f_{0}\pp y-\nabla f_{0}\pp z=c\pp{y-z}.$ Then,
we have that
\begin{eqnarray*}
\Verts{\nabla f\pp y-\nabla f\pp z}_{2} & = & \Verts{\sum_{i=1}^{N}a_{i}\phi'\pp{a_{i}y}-\sum_{i=1}^{N}a_{i}\phi'\pp{a_{i}z}+c\pp{y-z}}_{2}\\
 & = & \Verts{\sum_{i=1}^{N}a_{i}\pars{\phi'\pp{a_{i}y}-\phi'\pp{a_{i}z}}+c\pp{y-z}}_{2}\\
 & = & \Verts{\sum_{i=1}^{N}a_{i}\int_{a_{i}^{\top}z}^{a_{i}^{\top}y}\phi''\pp tdt+c\pp{y-z}}_{2}\\
 & = & \Verts{\sum_{i=1}^{N}a_{i}\pars{a_{i}^{\top}y-a_{i}^{\top}z}b_{i}+c\pp{y-z}}_{2}\\
 &  & -\theta\leq b_{i}\leq\theta\\
 & = & \Verts{\pars{cI+\sum_{i=1}^{N}b_{i}a_{i}a_{i}^{\top}}\pars{y-z}}_{2}\\
 & \leq & \Verts{\pars{cI+\theta\sum_{i=1}^{N}a_{i}a_{i}^{\top}}\pars{y-z}}_{2}.
\end{eqnarray*}

The final inequality is justified by the following claim: $\Verts{\sum_{i=1}^{N}b_{i}a_{i}a_{i}^{\top}\pars{y-z}}_{2}^{2}$
is maximized over vectors $b$ with $0\leq b_{i}\leq\theta$ by setting
$b_{i}=\theta$ always. To establish this claim observe that
\begin{eqnarray*}
\frac{d}{db_{k}}\Verts{\pars{cI+\sum_{i=1}^{N}b_{i}a_{i}a_{i}^{\top}}\pars{y-z}}_{2}^{2} & = & \frac{d}{db_{k}}\pars{\pars{cI+\sum_{i=1}^{N}b_{i}a_{i}a_{i}^{\top}}\pars{y-z}}^{\top}\pars{\pars{cI+\sum_{j=1}^{N}b_{j}a_{j}a_{j}^{\top}}\pars{y-z}}\\
 & = & 2\pars{\pars{cI+\sum_{i=1}^{N}b_{i}a_{i}a_{i}^{\top}}\pars{y-z}}^{\top}\frac{d}{db_{k}}\pars{cI+\sum_{i=1}^{N}b_{i}a_{i}a_{i}^{\top}}\pars{y-z}\\
 & = & 2\pars{y-z}^{\top}\pars{cI+\sum_{i=1}^{N}b_{i}a_{i}a_{i}^{\top}}\pars{cI+b_{k}a_{k}a_{k}^{\top}}\pars{y-z}\\
 & = & 2\tr\pars{cI+\sum_{i=1}^{N}b_{i}a_{i}a_{i}^{\top}}\pars{cI+b_{k}a_{k}a_{k}^{\top}}\pars{y-z}\pars{y-z}^{\top}\\
 & = & 2\tr\pars{cI+\sum_{i=1}^{N}b_{i}a_{i}a_{i}^{\top}}b_{k}a_{k}a_{k}^{\top}\pars{y-z}\pars{y-z}^{\top}\\
 &  & +2c\tr\pars{cI+\sum_{i=1}^{N}b_{i}a_{i}a_{i}^{\top}}\pars{y-z}\pars{y-z}^{\top}\\
 & = & 2b_{k}\tr a_{k}^{\top}\pars{y-z}\pars{y-z}^{\top}\pars{cI+\sum_{i=1}^{N}b_{i}a_{i}a_{i}^{\top}}a_{k}\\
 &  & +2c\tr\pars{y-z}^{\top}\pars{cI+\sum_{i=1}^{N}b_{i}a_{i}a_{i}^{\top}}\pars{y-z}\\
 & \geq & 0.
\end{eqnarray*}
The last inequality follows from the fact that
\[
\pars{cI+\sum_{i=1}^{N}b_{i}a_{i}a_{i}^{\top}}
\]
and
\[
\pars{y-z}\pars{y-z}^{\top}
\]
are both real, symmetric positive definite matrices.
\end{proof}

\clearpage{}

\newpage{}

\section{Specific Models}

\subsection{Linear Model}

Suppose that $p\pp{\z}=\N\pp{\z\vert0,\frac{1}{c}I}$ and $p\pp{y_{i}|\b x_{i},\z}=\N\pp{y_{i}\vert\z^{\top}\b x_{i},\frac{1}{b}}.$
Then, we have that
\begin{eqnarray*}
p\pp{\z}\prod_{i}p\pp{y_{i}\vert\b x_{i},\z} & \propto & \exp\pars{-\frac{1}{2c}\Verts{\z}^{2}-\sum_{i}\frac{1}{2b}\pp{y_{i}-\z^{\top}\b x_{i}}^{2}}\\
 & = & \exp\pars{-\frac{c}{2}\Verts{\z}^{2}-\sum_{i}\frac{b}{2}\pp{y_{i}-\z^{\top}\b x_{i}}^{2}}\\
 & = & \exp\pars{-\frac{c}{2}\Verts{\z}^{2}-\frac{b}{2}\Verts{\b y-X\z}_{2}^{2}}\\
 & = & \exp\pars{-\frac{c}{2}\Verts{\z}^{2}-\frac{b}{2}\Verts{\b y}_{2}^{2}+b\b y^{\top}X\z-\frac{b}{2}\z^{\top}X^{\top}X\z}\\
 & \propto & \exp\pars{b\b y^{\top}X\z-\frac{1}{2}\z^{\top}\pars{bX^{\top}X+cI}\z}\\
 & = & \exp\pars{\b a^{\top}\z-\frac{1}{2}\z^{\top}\Sigma^{-1}\z}\\
 & \propto & \exp\pars{-\frac{1}{2}\pp{\z-\Sigma\b a}\Sigma^{-1}\pp{\z-\Sigma\b a}}\\
 & = & \exp\pars{-\frac{1}{2}\pp{\z-\mu}\Sigma^{-1}\pp{\z-\mu}}
\end{eqnarray*}
\begin{eqnarray*}
\Sigma & = & \pars{bX^{\top}X+cI}^{-1}\\
\mu & = & \Sigma\b a\\
 & = & \pars{bX^{\top}X+cI}^{-1}bX^{\top}\b y\\
 & = & \pars{X^{\top}X+\frac{c}{b}I}^{-1}X^{\top}\b y
\end{eqnarray*}

\section{Reparameterization Stuff}

\subsection{Motivation}

Suppose that $\log p(z,x)$ is something of the form
\[
\log p\pp{z,x}=1^{\top}\phi\pp{Xz}.
\]
We have that
\[
\nabla_{z}\log p\pp{z,x}=X^{\top}\phi'\pp{Xz}
\]
and that
\[
\nabla_{z}^{2}\log p\pp{z,x}=X^{\top}\phi''\pp{Xz}X.
\]
If we suppose that $0\leq\phi''\leq\theta$ (for example this is true
with logistic regression with $\theta=\frac{1}{4}$) then we have
that
\[
0\preceq\nabla_{z}^{2}\log p\pp{z,x}\preceq\theta X^{\top}X.
\]
If we were to add a uniform prior, we'd have something like
\[
cI\preceq\nabla_{z}^{2}\log p\pp{z,x}\preceq cI+\theta X^{\top}X.
\]

On the other hand, for Bayesian regression, we'd have something like
\[
\theta X^{\top}X\preceq\nabla_{z}^{2}\log p\pp{z,x}\preceq\theta X^{\top}X
\]
with $\theta=1$. This offers much stronger possibilities for rescaling.

\subsection{Divergence}

Suppose that $\log p(z,x)$ is some distribution that is ``poorly
scaled''. That is, if we compute the condition number, it is quite
poor. On the other hand, it could be that for some $A$ and $b$,
$\log p\pp{Az+b,x}$ is much better-conditioned. The following lemma
shows that we are free to re-scale $p$ in whatever way we want and
then have $q$ target that rescaled distribution. Once that's done,
we can then transform $q$ back to the original space.

\begin{restatable}{lem1}{reparam}Suppose that $p_{\zr}\pp z$ is
some distribution and $p_{\r y}\pp y$ is the distribution of $A\zr+b,\ \zr\sim p_{\zr}$,
namely
\[
p_{\r y}\pp y=\frac{1}{\verts A}p_{\zr}\pp{A^{-1}\pp{y-b}}.
\]
Suppose that $q_{\r y}$ is some distribution which is ``close''
to $p_{\r y}$. If we define
\[
q_{\zr}\pp z=\verts Aq_{\r y}\pp{Az+b},
\]
then $\KL{q_{\zr}}{p_{\zr}}=\KL{q_{\r y}}{p_{\r y}}.$\end{restatable}

\subsection{Concrete}

\begin{restatable}{lem1}{posdefproduct}If $B\preceq C$ then $A^{\top}BA\preceq A^{\top}CA.$\end{restatable}

\subsection{Proofs}

\reparam*
\begin{proof}
In more detail, we know that if $\r y=T\pp{\r z}$then $\P\pp{\zr=z}=\P\pp{\r y=T\pp z}\verts{T'\pp z}$.
In our case, we use $T\pp z=Az+b$ so we have that
\[
p_{\r z}\pp{z,x}=p_{\r y}\pp{Az+b,x}\verts A
\]

Intuitively, we should correspondingly define
\[
q_{\zr}\pp z=q_{\r y}\pp{Az+b,x}\verts A.
\]
Then, we have that
\begin{eqnarray*}
\E_{\zr\sim q_{\zr}}\log\frac{p_{\zr}\pp{\zr,x}}{q_{\zr}\pp{\zr,x}} & = & \E_{\zr\sim q_{\zr}}\log\frac{p_{\r y}\pp{Az+b,x}\verts A}{q_{\r y}\pp{Az+b,x}\verts A}\\
 & = & \int q_{\zr}\pp z\log\frac{p_{\r y}\pp{Az+b,x}}{q_{\r y}\pp{Az+b,x}}dz\\
 & = & \int\verts Aq_{\r y}\pp{Az+b,x}\log\frac{p_{\r y}\pp{Az+b,x}}{q_{\r y}\pp{Az+b,x}}dz\\
 & = & \int q_{\r y}\pp{y,x}\log\frac{p_{\r y}\pp{y,x}}{q_{\r y}\pp{y,x}}dy
\end{eqnarray*}
Where in the last line we apply 
\[
\int f\pp ydy=\int f\pars{T\pp z}\verts{\nabla T\pp z}dz
\]
with $f\pp y=q_{\r y}\pp{y,x}\log\frac{p_{\r y}\pp{y,x}}{q_{\r y}\pp{y,x}}$
and $T\pp z=Az+b.$
\end{proof}
\posdefproduct
\begin{proof}
Suppose that $B\preceq C$ meaning that $C-B$ is positive definite.
Then note that

\[
A^{\top}CA-A^{\top}BA=A^{\top}\pp{C-B}A
\]
is also positive definite, since for any $x$,
\begin{eqnarray*}
x^{\top}A^{\top}\pp{C-B}Ax & = & z^{\top}\pp{C-B}z,\ \ \ z=Ax.\\
 & \geq & 0.
\end{eqnarray*}

Thus we have that
\[
A^{\top}BA\preceq A^{\top}CA.
\]
\end{proof}

\section{Gradient Variance with a Full-Covariance Quadratic}

Suppose that $f\pp z=\frac{1}{2}\pp{\z-\zo}^{\top}M\pp{\z-\zo}.$
What is the gradient variance? The gradient is $\nabla f\pp{\z}=M\pp{\z-\zo}.$
Thus, we seem to get that
\begin{eqnarray*}
\E_{\ur\sim s}\Verts{\nabla_{\b w}f\pp{\T_{\b w}\pp{\ur}}}_{2}^{2} & = & \E\Verts{\nabla f\pp{\T_{\b w}\pp{\ur}}}_{2}^{2}\pars{1+\Verts{\ur}_{2}^{2}}\\
 & = & \E\Verts{M\ \pp{\T_{\b w}\pp{\ur}-\zo}}_{2}^{2}\pars{1+\Verts{\ur}_{2}^{2}}\\
 & = & \E\Verts{MC\ur+\b m-M\zo}_{2}^{2}\pars{1+\Verts{\ur}_{2}^{2}}\\
 & = & \pars{d+1}\Verts{\b m-M\zo}_{2}^{2}+\pars{d+\E\bb{\ur_{1}^{4}}}\Verts{MC}_{F}^{2}.
\end{eqnarray*}
The key thing, for this to work is showing that
\[
\Verts{\nabla f\pp y-\nabla f\pp z}_{2}\leq\Verts{M\pp{y-z}}_{2}.
\]
Certainly, if we had a property like that, we would be in business.

Claim: If $f$ is $M$-smooth in the above sense, then $\frac{1}{2}\z^{\top}M\z-f\pp{\z}$
is convex.

What does the above say about the Hessian? For very close $y$ and
$z$,
\[
\nabla f\pp y-\nabla f\pp z\approx\nabla^{2}f\pp z\pp{y-z}.
\]
Thus the bound sort of says that
\[
\Verts{\nabla^{2}f\pp z\pp{y-z}}_{2}^{2}\leq\Verts{M\pp{y-z}}_{2}^{2}.
\]
Or, essentially, that
\[
x^{\top}\pars{\nabla^{2}f\pp z}^{2}x\leq x^{\top}M^{2}x.
\]

\end{document}